\documentclass[11pt]{article}
\pdfoutput=1
\usepackage{fullpage}
\usepackage[margin=1in]{geometry}
\usepackage[utf8]{inputenc}

\usepackage{graphicx}
\usepackage{booktabs}
\usepackage{multirow}
\usepackage{multicol}

\usepackage{amsmath}
\usepackage{amssymb}
\usepackage{amsthm}

\usepackage{thmtools}
\usepackage{thm-restate}
\usepackage{bbm}
\usepackage{bm}
\usepackage{verbatim}
\usepackage[pagebackref]{hyperref}
\hypersetup{colorlinks=true,citecolor=blue,linkcolor=blue}
\usepackage{array}
\usepackage{caption}
\usepackage{subcaption}
\usepackage{float}
\usepackage{changepage}
\usepackage{enumitem}
\usepackage{cases}
\usepackage[dvipsnames]{xcolor}
\usepackage{algorithm}
\usepackage{algpseudocode}
\usepackage{cleveref}

\usepackage{xfrac}
\usepackage{tikz, pgfplots}

\usetikzlibrary{decorations.pathreplacing}

\usepackage{natbib}

\newcommand{\R}{\mathbb{R}}

\renewcommand{\hat}{\widehat}

\DeclareMathOperator{\E}{\mathbb{E}}

\newcommand{\state}{y}

\newcommand{\prob}{p}

\newcommand{\expect}[2]{{\E}_{#1}\left[#2\right]}
\newcommand{\var}[2]{{\textsc{Var}}_{#1}\left[#2\right]}

\newcommand{\ECE}{\textsc{ECE}}

\newcommand{\ind}[1]{\mathbb{I}\left[#1\right]}

\newcommand{\pred}{\prob}

\newcommand{\vpred}{\bm{\pred}}
\newcommand{\vstate}{\bm{\state}}

\newcommand{\distcal}{\underline{\textsc{distCal}}}

\newcommand{\Ddistcal}{\textsc{distCal}}

\newcommand{\smooth}{\textsc{smCal}}
\newcommand{\intc}{\textsc{intCal}}

\newcommand{\CAL}{\textsc{Cal}}

\newcommand{\referr}{\textsc{Ref}}

\newcommand{\intvset}{\mathcal{\intv}}
\newcommand{\intv}{I}

\newcommand{\binECE}{\textsc{BinECE}}
\newcommand{\binECEtwo}{\ell_2\textup{-}\textsc{BinECE}}

\newcommand{\bin}{B}

\newcommand{\atb}{\mathrm{ATB}}
\newcommand{\atbone}{\ell_1\textup{-}\mathrm{ATB}}
\newcommand{\sign}{\mathsf{sign}}

\newcommand{\binparam}{b}

\newcommand{\distribution}{D}

\newcommand{\report}{r}
\newcommand{\vreport}{\bm{r}}

\newcommand{\feature}{X}

\newcommand{\bernoulli}{\text{Ber}}

\newcommand{\cB}{\mathcal B}
\newcommand{\accp}{\mathsf{accP}}

\newcommand{\Z}{\mathbb Z}

\newcommand{\calF}{F}
\newcommand{\calR}{\mathcal R}
\newcommand{\zset}{Z}
\newcommand{\calD}{\Gamma}
\newcommand{\N}{\mathbb N}
\newcommand{\sv}{\bm {s}}

\newcommand{\ECEtwo}{\ell_2\text{-}\ECE}

\newcommand{\ubse}{\text{UBSE}}

\newcommand{\stepfunc}{h}

\newcommand{\distcaltwo}{\ell_2\text{-}\distcal}

\newtheorem{theorem}{Theorem}[section]
\newtheorem*{theorem*}{Theorem}
\newtheorem{lemma}[theorem]{Lemma}
\newtheorem{definition}[theorem]{Definition}

\newtheorem{proposition}[theorem]{Proposition}
\newtheorem{corollary}[theorem]{Corollary}

\newtheorem{observation}[theorem]{Observation}

\newtheorem{remark}[theorem]{Remark}
\newtheorem{claim}[theorem]{Claim}
\newtheorem{example}[theorem]{Example}
\newtheorem*{example*}{Example}

\title{A Perfectly Truthful Calibration Measure}
 \author{    
 Jason Hartline\\
     Northwestern University\\
     Computer Science\\
     \texttt{hartline@northwestern.edu}
     \and 
     Lunjia Hu\thanks{This work was done in part when Lunjia Hu was a Postdoctoral Fellow at Harvard University, supported by the Simons Foundation Collaboration on the Theory of Algorithmic Fairness and the Harvard Center for Research on Computation and Society (CRCS).} \\
     Northeastern University\\
     Khoury College of Computer Sciences\\
     \texttt{lunjia@alumni.stanford.edu}
     \and 
         Yifan Wu\thanks{The work was done when Yifan Wu was a PhD student at Northwestern University. }\\
     Microsoft Research, New England\\
     \texttt{yifan.wu2357@gmail.com}
     }
\date{}

\allowdisplaybreaks
\begin{document}

\maketitle
\begin{abstract}
   Calibration requires that predictions are conditionally unbiased and, therefore, reliably interpretable as probabilities. A calibration measure quantifies how far a predictor is from perfect calibration. 
 As introduced by \cite{haghtalab2024truthfulness}, a calibration measure is truthful if it is minimized in expectation when a predictor outputs the ground-truth probabilities. 
   Predicting the true probabilities guarantees perfect calibration, but in reality, when calibration is evaluated on a random sample,
   all known calibration measures incentivize predictors to lie in order to appear more calibrated. 
Such lack of truthfulness motivated
\citet{haghtalab2024truthfulness} and \citet{qiao2025truthfulness} to construct \emph{approximately} truthful calibration measures in the sequential prediction setting, but no \emph{perfectly} truthful calibration measure was known to exist even in the more basic batch setting.

    We design a simple, perfectly and strictly truthful, sound and complete calibration measure in the batch setting: averaged two-bin calibration error (ATB). 
   ATB is quadratically related to two existing calibration measures: the smooth calibration error $\smooth$ and the lower distance to calibration $\distcal$.
   The simplicity in our definition of ATB makes it efficient and straightforward to compute, 
   allowing us to give the first linear-time calibration testing algorithm, improving a result of \cite{hutesting}. We also introduce a general recipe for constructing truthful measures based on the variance additivity of independent random variables, which proves the truthfulness of ATB as a special case and allows us to construct other truthful calibration measures such as quantile-binned $\ell_2$-ECE.  
   
\end{abstract}
\thispagestyle{empty}
\newpage
\setcounter{page}{1}


\section{Introduction}

Probabilistic forecasting has become increasingly important in modern AI-assisted decision-making. Unlike deterministic classification, probabilistic forecasts provide uncertainty quantification, allowing assessment of risks. One desired property of probabilistic predictions is \textit{calibration}, which requires predictions to be conditionally unbiased and, therefore, reliably interpretable as probabilities. 
For example, neural networks for tumor diagnosis are trained to output a prediction $r\in [0, 1]$, ideally interpretable as the expectation of a binary state $y\in \{0, 1\}$: the tumor segment being malignant or not. 
The neural network is calibrated if, conditioned on the output $r$ being, say $40\%$, the probability that the tumor is malignant $\Pr[y = 1|r = 40\%]$ is also $40\%$.

 A calibration measure quantifies how far a predictor is from perfect calibration. 
 The Expected Calibration Error ($\ECE$) is a canonical calibration measure proposed by \citet{foster1997calibrated}. Given an empirical distribution of predictions and states, if conditioned on a reported prediction $\report\in [0,1]$, the actual empirical frequency of the state $y = 1$ is $\hat{r} := \Pr[\state = 1 | \report]$, then the absolute bias in prediction is $|\report - \hat{r}|$. $\ECE$ is defined as the expected bias in predictions, $\expect{\report}{|\report - \hat{r}|}$. 

A recent line of work studies the \emph{truthfulness} of calibration measures \citep{haghtalab2024truthfulness, qiao2025truthfulness}. An error measure is truthful if it incentivizes a predictor to output the truth, i.e., the expected error is minimized when the predictor reports the true probabilities. 
However, no known calibration measure is truthful.
Even a miscalibrated predictor can have lower expected error than the truthful predictor when evaluated by known calibration measures. We explain this non-truthfulness in \Cref{sec: intro non truthful of existing}. 
%

\subsection{Non-truthfulness of Known Calibration Measures}
\label{sec: intro non truthful of existing}

For existing non-truthful calibration measures, there exists obvious and uninformative prediction strategies that induce a lower calibration error than truth-telling. Concretely, the ``obvious'' strategy is to ignore all features and always predict the base rate, i.e., output the same constant for every input. In this section, we explain why this obvious strategy has a strictly lower error in expectation, and a lower loss even for every realization of samples. 
Following this observation, it is well-known that calibration is a poor measure of accuracy \citep{degroot1983comparison}.  

\Cref{example: intro nontruthful} illustrates the non-truthfulness of $\ECE$. 
We consider the batch setting: a sample of $T$ individuals whose binary states $\vstate = (y_1,\ldots,y_T)\in \{0,1\}^T$ are independently drawn from the Bernoulli distributions with means $\vpred = (p_1,\ldots,p_T)\in [0,1]^T$ (denoted by $\vstate\sim \vpred$). 
We say $\vpred = (p_1,\ldots,p_T)$ are the ground-truth probabilities.
In the example below, a predictor strictly benefits from reporting the base rate of ground truths. 
\begin{example}[ECE is 
not truthful, c.f.\ \citealp{sidestep}]
\label{example: intro nontruthful}
Suppose the ground-truth probabilities  are $\vpred = (\pred_1, \dots, \pred_T)$ where each $p_t$ is distributed independently and uniformly from $[1/3, 2/3]$. 
    An uninformative predictor that always predicts $r_1 = \cdots = r_T = 0.5$ achieves an expected empirical $\ECE=O\left( \sqrt{1/T}\right)$, the sampling error. However, a truthful predictor who reports $r_t = p_t$ results in a higher empirical $\ECE \ge 1/3$. This is because the predictions $r_1,\ldots, r_T\in [1/3,2/3]$ are almost surely distinct, so the empirical conditional expectation $\hat r_t := \E_{(r,y)\sim \mathrm{Unif}((r_t,y_t)_{t\in [T]})}[y|r = r_t]$ is simply $y_t\in \{0,1\}$, giving $|r_t - \hat r_t| = |r_t - y_t| \ge 1/3$.
    
\end{example}


\Cref{example: intro nontruthful} shows that the obvious uninformative prediction achieves a lower $\ECE$. 
Even worse, there are \emph{miscalibrated} predictors (e.g.\ the predictor that always predicts $0.5 + \epsilon$ for a small $\epsilon > 0$) achieving smaller ECE 
than the \emph{calibrated} truthful predictor.
Thus ECE does not rank predictors correctly based on how calibrated they are.


Predicting the uninformative base rate incurs a lower calibration error for all known calibration measures. 
It happens not just for $\ECE$ in \Cref{example: intro nontruthful} and its variants\footnote{Variants of $\ECE$ include $\ell_\alpha$-$\ECE$, where we replace the absolute bias $|\report - \hat{\report}|$ with $|\report - \hat{\report}|^\alpha$ for an arbitrary $\alpha \ge 1$, as well as binned versions of $\ell_\alpha$-$\ECE$.}, but also for continuous calibration measures\footnote{These are calibration measures that are continuous as a function of the predictions. Note that $\ECE$ and binned ECE are not continuous.} such as the smooth calibration error \citep{smooth}, the distance to calibration \citep{utc} and its variants, etc, 
irrespective of the sample size $T$. Moreover, it happens consistently accross \emph{every} realization of the states $\vstate$, \emph{not just in expectation}. Specifically, for \emph{every} realization of $\vstate \in \{0,1\}^T$  and \emph{every} prediction sequence $\vreport = (\report_1,\ldots,\report_T)\in [0,1]^T$, all these calibration measures $\CAL$ satisfy
\begin{equation}
\label{eq:intro-avg}
\CAL(\vreport, \vstate) \ge \CAL(\bar \vreport, \vstate),
\end{equation}
where $\bar \vreport = (\bar \report,\ldots,\bar\report)$ 
is the constant predictor that always predicts the average $\bar \report := \frac 1T \sum_{t = 1}^T \report_t$. 
This obvious and uninformative strategy always achieves (weakly) lower calibration error. For many realizations of states $\vstate$, the error is strictly lower. 
We formally prove this observation in \Cref{thm:avg}.

\subsection{Our Goal: Perfectly Truthful Calibration Measures}
Measuring and optimizing for calibration \textbf{non-truthfully} only makes the predictions \textbf{less trustworthy}, going in the very opposite direction of the goal of calibration. 
Recall the definition of truthfulness:
a calibration measure $\CAL$ is truthful if for every ground-truth probabilities $\vpred\in [0,1]^T$, the expected empirical calibration error $
\E_{\vstate\sim\vpred}[\CAL(\vreport,\vstate)]$ of predictions $\vreport\in [0,1]^T$ on a random sample $\vstate\sim \vpred$ is minimized when $\vreport = \vpred$.
From a machine learning perspective, 
 a truthful measure helps 
 identify the Bayes optimal predictor \citep{gneiting2011making} because it correctly ranks ground-truth predictions with the lowest expected error. 
From a game-theoretic perspective, a truthful measure incentivizes an optimizing predictor to output their true beliefs, where we view $\vpred$ as the predictor's subjective belief about the probabilities, which might differ from the true probabilities. 
For the tumor risk prediction task, if assessed by a non-truthful calibration measure, a doctor is incentivized to report a prediction different from their true probabilistic assessment of the tumor risk, to make the predictions look more calibrated. Such an incentivized misreport can hardly be trusted. 

We focus on the perfect truthfulness in the batch setting. Previous work \citep{haghtalab2024truthfulness, qiao2025truthfulness} design approximately truthful calibration measures in the sequential prediction setting, where the states $y_1,\ldots,y_T$ are revealed sequentially after each prediction $r_t$ is made. 
We observe that, in the simpler batch setting, some existing measures are approximately truthful, such as the smooth calibration error \citep{smooth, utc} and the calibration measures proposed by \citet{haghtalab2024truthfulness} and \citet{qiao2025truthfulness}. Yet, no known calibration measure is perfectly truthful.

The two minimum requirements of a calibration measure are \emph{completeness} and \emph{soundness}. A calibration measure should 
be able to distinguish calibrated predictors from miscalibrated ones given a sufficiently large sample. 
Completeness requires vanishing error when a predictor is calibrated (\Cref{def:consistent}), and soundness requires non-vanishing error when a predictor is miscalibrated as the sample size increases (\Cref{def:sound}).  \citet{haghtalab2024truthfulness} point out that some error metrics, such as the well-known squared error $\frac{1}{T}\sum_t (\report_t  - \state_t)^2$, are truthful but far from being a complete and sound calibration measure. The squared error of a calibrated predictor 
may not vanish as the sample size $T$ increases.


Our main result shows that truthfulness can be achieved via surprisingly simple constructions in the batch setting, while preserving the completeness and soundness of existing calibration measures. 

\subsection{A Roadmap of Our Contributions}


As our main contribution, we construct a perfectly truthful, complete and sound calibration measure: Averaged Two-Bin Calibration Error ($\atb$). In addition, ATB maintains and even significantly improves other desired properties considered in the literature: 
continuity and consistency \citep{utc}, low sample complexity, and high computational efficiency. 


We propose a general framework for constructing truthful error metrics, where ATB is a special case. 
This general framework allows us to construct other truthful calibration measures while preserving existing properties. For example, we construct the truthful quantile-binned $\ell_2$-ECE with surprisingly simple changes to binned ECE. 


As bonus side products of our simple truthful construction, we give the first linear-time algorithm for the calibration testing problem studied by \cite{hutesting} as well as a simple, faster constant-factor approximation to the smooth calibration error and the Distance to Calibration\footnote{Following \citet{arunachaleswaran2025elementary, qiao-distance}We refer to the computable Lower Distance to Calibration in \citet{utc} as the Distance to Calibration. }. 

\paragraph{Perfectly Truthful Calibration Measure.} 
\Cref{sec:atb} introduces ATB, constructed with two bins and a randomized binning boundary: 
\begin{itemize}
    \item The prediction space of $[0, 1]$ is divided into two bins $[0,q)$ and $[q,1]$, with the bin boundary $q$ chosen uniformly at random from $[0,1]$.
    \item Within each bin, we compute the squared error between the sum of the predictions and the sum of the states. 
    \item After summing up the errors in the two bins and dividing the result by $T^2$, we define $\atb$ to be the expectation over the random choice of the bin boundary $q$.
\end{itemize}

\begin{definition}[Averaged two-bin calibration error]
\label{def: intro atb}
Given predictions $\vreport = (\report_1,\ldots,\report_T)\in [0,1]^T$ and states $\vstate = (\state_1,\ldots,\state_T)\in \{0,1\}^T$, we define
\begin{align*}
\atb(\vreport,\vstate) & := \E_{q\sim \mathrm{Unif}([0,1])}\left[\frac 1{T^2} \left(\left(\sum_{t: \report_t < q} (\report_t - \state_t)\right)^2 + \left(\sum_{t: \report_t \ge q}(\report_t - \state_t)\right)^2\right)\right].
\end{align*}
\end{definition}

We prove the following properties of ATB.
\begin{itemize}
    \item Perfect truthfulness (\Cref{thm: atb truthful}). We additionally show that ATB is \emph{strictly} truthful, meaning that truth-telling is the only minimizer of the expected error, in the ex-ante stage where a predictor is evaluated on a random sample of $T$ i.i.d.\ individuals. 
    \item Lipschitz continuity in predictions. For every two prediction sequences 
$\vreport_1,\vreport_2\in [0,1]^T$,
\[
|\atb(\vreport_1, \vstate) - \atb(\vreport_2, \vstate)| \le \frac 6T \|\vreport_1 - \vreport_2\|_1.
\]


\item Completeness and soundness. The completeness and soundness of ATB follow from the quadratic relationship to 
two known continuous, complete and sound calibration measures: the smooth calibration error $\smooth$ \citep{smooth} and the lower distance to calibration $\distcal$ \citep{utc}. See \Cref{thm: complete sound}. 





\item Sample and computational complexity. 
ATB can be estimated within $\varepsilon$ error using $T = O(\varepsilon^{-2})$ examples in $O(T)$ time.
For exact computation of ATB, \Cref{thm:compute} shows an $O(T\log T)$-time algorithm on $T$ examples.
For approximate computation, \Cref{thm:compute} shows an $O(T + 1/\varepsilon)$-time algorithm for approximating up to additive error $\varepsilon > 0$. For $\varepsilon\approx 1/\sqrt T$, the running time is $O(T + 1/\varepsilon) = O(T)$.

These algorithms are simple and easy to implement, following the simplicity of the definition of ATB. 
For exact computation, our algorithm sorts the predictions in $O(T\log T)$ time, and does the rest of the computation in $O(T)$ time. 
For $\varepsilon$-additive approximation, 
the sorting step can be implemented in $O(T + 1/\varepsilon)$ time by discretizing the predictions before sorting.
For comparison, the best-known algorithms for computing $\smooth$ and $\distcal$ are much more complicated and take $O(T\log ^2 T)$ and $O(T^2 \log T)$ time, respectively, even when $O(1/\sqrt T)$ additive error is allowed \citep{hutesting}.

\end{itemize}

\paragraph{General Framework for Constructing Truthful Calibration Measures.} 
We establish the truthfulness of ATB by showing that it is a member of a general family of truthful error metrics, which we term the Unnormalized Binned Squared Errors (UBSEs) (\Cref{sec:ubse}).


We define UBSEs as binning-based calibration measures where the binning scheme can be randomized and can depend arbitrarily on the predictions $\vreport$. The key to its truthfulness lies in how the biases in each bin are combined to calculate the final error. We compute the squared biases in each bin similarly to the standard $\ell_2$ ECE, but there is a subtle but crucial difference in how these squared biases are weighted and combined. We discuss this family in more detail in \Cref{sec:tech-truthful} and \Cref{sec:ubse}. We prove that each UBSE has the following key property of error decomposition, which implies truthfulness. 

\begin{lemma}[Informal, \Cref{lem: error decomposition}]
\label{lem: intro error decomposition}
\begin{align*}
\underbrace{\E_{\vstate \sim \vpred}[\ubse(\vreport,\vstate)]}_{\text{Empirical $\ubse$}} \quad = \underbrace{\ubse(\vreport,\vpred)}_{\text{$\ubse$ on the true distribution}}  + \underbrace{\frac 1{T^2}\sum_{t = 1}^T p_t(1 - p_t).}_{\text{ Variance of avg.\ of $\vstate$ (invariant of the predictions $\vreport$) }}
\end{align*}
\end{lemma}

To see the truthfulness, when $\vreport = \vpred$, the middle term is $0$, while the third term is invariant of the predictions. The expected error is thus minimized at $\vreport = \vpred$. 

Our construction of $\atb$ is a special case of $\ubse$, and thus the truthfulness of $\atb$  
follows immediately from \Cref{lem: intro error decomposition}.
We remark that while every UBSE is truthful, it may not have the other desired properties of ATB, such as soundness, continuity, and the relationship to $\smooth$ and $\distcal$.

To illustrate the idea behind our definition  of the UBSEs, we provide a technical overview in \Cref{sec:tech-truthful} and demonstrate, as an example, how to \textbf{make binned ECE truthful} using a quantile-based binning scheme, giving another truthful calibration measure \emph{Quantile $\ell_2$-BinECE}.



\paragraph{Linear-Time Calibration Test via Validity.} We introduce a quantitative framework for evaluating the validity of a calibration error using calibration tests. As the sample size increases, the notions of completeness and soundness require a calibration error to vanish  given a calibrated predictor, and not vanish given a mis-calibrated predictor.
We define the \textit{validity} of a calibration error using its ability to distinguish calibration from mis-calibration, which can be viewed as a quantitative combination of completeness and soundness (see \Cref{def: validity}). Any valid calibration measure implies an algorithm for solving the \textit{calibration testing} problem in \citet{hutesting}. 

%
We show that our $\atb$ is optimally valid for the smooth calibration error and the lower distance to calibration. 
\begin{theorem}[Informal, see \Cref{thm: optimal validity}]
\label{thm:intro-test}
    Given $T$ examples, $\atb$ can distinguish a calibrated predictor from a predictor with $\smooth \ge C/\sqrt T$ for an absolute constant $C > 0$. Moreover, this rate is information-theoretically optimal (up to constant factors).
    This also holds when we replace $\smooth$ with $\distcal$ because the two are constant-factor approximations of each other as shown by \cite{utc}.
\end{theorem}

As mentioned earlier, ATB can be computed in $O(T\log T)$ time. Thus, \Cref{thm:intro-test} implies a faster calibration tester than the previous $O(T\log ^2 T)$-time calibration tester for $\smooth$ and $\distcal$ from \cite{hutesting}. Moreover, the running time can be further improved to $O(T)$ because we do not need to compute ATB exactly and just need to approximate it up to error $O(1/T)$ (see \Cref{sec:test}). This gives the first linear-time algorithm for calibration testing.

\paragraph{Simple and Efficient Constant-factor Approximation of $\smooth$ and $\distcal$.} 
As mentioned above, the completeness and soundness are proven by the quadratic relationship of ATB to $\smooth, \distcal$. This relationship is established via the following non-truthful $\ell_1$ variant of $\atb$, which quadratically approximates $\atb$. We show that this variant gives a \emph{constant-factor} approximation for $\smooth$ and $\distcal$.
\begin{definition}
\label{def: atb1}
We define the $\ell_1$ variant of $\atb$:
    \begin{align*}
        \atbone(\vreport,\vstate) &: = \E_{q\sim \mathrm{Unif}([0,1])}\left[\frac 1T \left(\left|\sum_{t: \report_t < q} (\report_t - \state_t)\right| + \left|\sum_{t: \report_t \ge q}(\report_t - \state_t)\right|\right)\right].\\
    \end{align*}
\end{definition}

\begin{theorem}[Informal, \Cref{cor:relationship}]
\label{thm: intro atb1 constant approx}
    $\atbone$ is a constant-factor approximation to $\smooth$ and $\distcal$:
    \begin{align*}
        \frac 13\, \distcal(\vreport,\vstate) & \le \atbone(\vreport,\vstate) \le 3\,\distcal(\vreport,\vstate),\\
                \frac 23\, \smooth(\vreport,\vstate) & \le \atbone(\vreport,\vstate) \le 6\,\smooth(\vreport,\vstate).
    \end{align*}
\end{theorem}
Based on \Cref{thm: intro atb1 constant approx}, the quadratic relationship between $\atb$ and $\smooth,\distcal$ follows from the relatively easy observation that $\atb$ and $\atbone$ are themselves quadratically related. 

Similar to $\atb$, $\atbone$ can also be easily computed in $O(T\log T)$ time and approximated up to $\varepsilon >0$ additive error in $O(T + 1/\varepsilon)$ time (see \Cref{sec:compute}). Therefore, \Cref{thm: intro atb1 constant approx} implies a faster constant-factor approximation algorithm for $\smooth$ and $\distcal$ than direct computation, for which the best-known algorithms take $O(T\log^2 T)$ and $O(T^2\log T)$ time, respectively \citep{hutesting}.

Previously, \citet{utc} also introduced a binning-based approximation to the distance to calibration, termed the \emph{interval calibration error}.
However, the definition of interval calibration error is more involved than $\atbone$. It requires optimizing the binning scheme (including the number of bins) to minimize the sum of the average bin width and the binned calibration error. Also, the interval calibration error only gives a \emph{quadratic} approximation for $\smooth$ and $\distcal$.
Our \Cref{thm: intro atb1 constant approx} shows, perhaps surprisingly, that using only $2$ bins suffices to give a \emph{constant-factor} approximation for $\smooth$ and $\distcal$ (see \Cref{sec:tech-atbone} for a technical overview).

\subsection{Technical Overview}
\label{sec: intro technical overview}
We give a high-level explanation for the two major technical ingredients we use to establish our results. The first is a general recipe for constructing truthful error metrics leveraging the variance additivity of independent random variables. The second is our analysis establishing the constant-factor approximation between $\atbone$ and the smooth calibration error.

\subsubsection{Truthfulness from Variance Additivity} 
\label{sec:tech-truthful}
We discuss the idea behind our construction of a general family of truthful measures, i.e.,\ Unnormalized Binned Squared Errors (UBSEs). As mentioned earlier, ATB is a member of this family, so its truthfulness follows as a consequence.

Our goal is to measure the calibration error of a sequence of predictions $\vreport = (r_1,\ldots,r_T)\in [0,1]^T$ on the states $\vstate = (y_1,\ldots,y_T)\in \{0,1\}^T$. 
Here, each state $y_t$ is sampled independently from the Bernoulli distribution with mean $p_t\in [0,1]$, where $\vpred = (p_1,\ldots,p_T)\in [0,1]^T$ are the true probabilities (denoted by $\vstate\sim \vpred$). 

Our first observation is that, if we divide the predictions $r_t$ into bins based on the \emph{indices} $t\in [T]$ rather than the \emph{values} $r_t\in [0,1]$, then truthfulness can be easily achieved by the $\ell_2$ version of binned ECE, $\binECEtwo$. Concretely, consider a fixed partition  $\cB = (B_1,\ldots,B_k)$ of the index space $[T]$ into bins: $[T] = B_1\cup \cdots \cup B_k$. The $\binECEtwo$ follows the standard computation of $\ECE$ but replacing the $\ell_1$ error with squared error:
\begin{align}
\binECEtwo_\cB(\vreport,\vstate) &= \sum_{i\in [k]} \underbrace{\frac{|B_i|}{T}}_{\text{weigh by fraction}}\cdot\left(\underbrace{\frac 1{|B_i|}}_{\text{normalize by size}} \underbrace{\sum_{t\in B_i}(r_t - y_t)}_{\text{the bias in $\bin_i$}}\right)^2  \nonumber\\
&=\sum_{i\in [k]} \frac{1}{T|B_i|}\left(\sum_{t\in B_i}(r_t - y_t)\right)^2. \label{eq:tech-1-0}
\end{align}

Assuming the index partition $\cB$ is fixed, the truthfulness of $\binECEtwo$ comes from the truthfulness of squared error: within each bin $B_i$, the expected squared bias over $\vstate\sim \vpred$
\begin{equation}
\label{eq:tech-1-0-1}
\E_{\vstate\sim\vpred}\left[\left(\sum_{t\in B_i}(r_t - y_t)\right)^2\right]
\end{equation}
is minimized if and only if $\sum_{t\in B_i}r_t = \sum_{t\in B_i}p_t$, implying minimized by predicting the truth $r_t = p_t$.

However, $\binECEtwo$ in its standard form is not truthful because it does not bin based on indices $t$, but rather the prediction values $r_t$. Specifically, $\binECEtwo$ combines adjacent predictions into the same bin. It works by first partitioning the prediction space $[0, 1]$ into intervals, with one bin corresponding to each interval. Each bin consists of the predictions $v_t$ that fall within the corresponding interval.  A strategic predictor will try to reduce the expected error by manipulating the partition $\cB$ via making untruthful predictions. For the same \Cref{example: intro nontruthful}, always predicting the same value $r_1 = \cdots = r_T$ puts all the indices in the same bin, resulting in a different index partition $\cB'$ than predicting truthfully. This different partition can significantly reduce the expected $\binECEtwo$ due to the bin-size-based normalization $1/|B_i|$ in \eqref{eq:tech-1-0}.

The example above hints that 
we can achieve truthfulness if the expected error of truthful predictions is \emph{invariant} to the index partition $\cB$. To see this, our analysis above shows that for any index partition $\cB'$ that could potentially be manipulated from a strategic report $\vreport$, predicting the truth achieves smaller or equal expected error on the same index partition $\cB'$: 
\begin{equation}
\label{eq:tech-1-1}
\E_{\vstate\sim \vpred}[\binECEtwo_{\cB'}(\vpred,\vstate)] \le \E_{\vstate\sim \vpred}[\binECEtwo_{\cB'}(\vreport,\vstate)].
\end{equation}
To establish truthfulness, we need to show that the truthful predictions $\vpred$ achieve smaller or equal expected error on the un-manipulated partition $\cB$ from truthful predictions: 
\begin{equation}
\label{eq:tech-1-2}
\E_{\vstate\sim \vpred}[\binECEtwo_{\cB}(\vpred,\vstate)] \le \E_{\vstate\sim \vpred}[\binECEtwo_{\cB'}(\vreport,\vstate)].
\end{equation}
To bridge the gap between what we have \eqref{eq:tech-1-1} and what we need \eqref{eq:tech-1-2}, it suffices if the expected truthful error is invariant of the binned partition for truthful predictions:
\begin{equation}
\label{eq:tech-1-invariance}
\E_{\vstate\sim \vpred}[\binECEtwo_{\cB'}(\vpred,\vstate)] = \E_{\vstate\sim \vpred}[\binECEtwo_{\cB}(\vpred,\vstate)] \quad\text{for any partitions $\cB,\cB'$}.
\end{equation}

We obtain \Cref{eq:tech-1-invariance} with our second key observation: we can modify the $\binECEtwo$ in \eqref{eq:tech-1-0} to achieve partition-invariance by removing the bin-size-based normalization $1/|\bin_i|$. To see this, let us compute the expected squared bias in each bin $B_i$ for truthful predictions ($r_t = p_t$):
\begin{equation}
\label{eq:tech-1-variance}
\E_{\vstate\sim\vpred}\left[\left(\sum_{t\in B_i}(p_t - y_t)\right)^2\right] = \var{\vstate\sim\vpred}{\sum_{t\in B_i}y_t} = \sum_{t\in B_i}\var{\vstate\sim\vpred}{y_t} = \sum_{t\in B_i} p_t(1 - p_t),
\end{equation}
Here we crucially use the variance additivity of independent random variables: since the $y_t$'s are independent, the variance of their sum is equal to the sum of their variances. If we directly add up \eqref{eq:tech-1-variance} over the bins $i = 1,\ldots,k$ without any bin-size-based normalization, we get
\[
\E_{\vstate\sim\vpred}\left[\sum_{i = 1}^k\left(\sum_{t\in B_i}(p_t - y_t)\right)^2\right] = \sum_{t\in [T]} p_t(1 - p_t),
\]
which is indeed invariant to the partition $\cB$, satisfying \eqref{eq:tech-1-invariance}. Therefore, our analysis shows that the following unnormalized calibration error is truthful, where now the partition $\cB = (B_1,\ldots,B_k)$ (including the choice of $k$) can (arbitrarily!) depend on the predictions $r_1,\ldots,r_T$:
\[
\CAL(\vreport,\vstate):= \frac k{T^2}\sum_{i = 1}^k\left(\sum_{t\in B_i}(r_t - y_t)\right)^2.
\]
Moreover, one can show that any calibrated predictions (not just truthful predictions) achieve the same expected error $\frac k{T^2}\sum_{t\in [T]} p_t(1 - p_t) = O\left(\frac{k}{T}\right)$ (see \Cref{lem: error decomposition}), which vanishes as $T\to \infty$, implying the completeness of the error (\Cref{def:consistent}) in addition to truthfulness. Note that the global normalization by $k/T^2$ ensures completeness while not affecting the truthfulness property, unlike the per-bin normalization in \eqref{eq:tech-1-0} depending on individual bin sizes $|B_i|$.

The final observation that completes our construction of UBSEs is that, by linearity of expectation, we can even allow the partition $\cB$ to be randomized and use the expected error over the random partition $\cB$, without breaking the truthfulness analysis above.
This allows us to construct a general family of truthful calibration errors, each using a different, possibility randomized, partition $\cB$ that can depend on the predictions $r_1,\ldots,r_T$.
We term these calibration errors \emph{Unnormalized Binned Squared Errors} (UBSEs) and present the formal definition and analysis in \Cref{sec:ubse}. We note that this family is not necessarily sound. The soundness depends on the binning strategy. Below, we present a sound and complete calibration error. 

\paragraph{Example: Quantile $\binECEtwo$ is truthful.} As an example of $\ubse$, simple modifications make $\binECEtwo$ truthful by binning predictions according to quantiles. 
With $k$ bins, the following $\ubse$ is truthful.
\begin{itemize}
    \item Sort the samples by reported predictions with $\report_1\leq \dots\leq \report_T$. Break ties uniformly at random. 
    \item Divide predictions into $k$ bins, with $\frac{T}{k}$ predictions in each bin.
    \item Calculate $\ubse$. 
\end{itemize}
Binning according to quantiles ensures that each bin contains the same number of predictions and thus, the normalization factors based on bin sizes $|B_i|$ in \eqref{eq:tech-1-0} no longer break truthfulness.

\subsubsection{Two-Bin Approximation of the Smooth Calibration Error}
\label{sec:tech-atbone}

Our UBSE framework is flexible with regard to how the bins should be chosen (including how many bins should be chosen). However, it is not obvious to find an appropriate binning scheme and show that the corresponding UBSE is polynomially related to existing calibration error metrics such as $\smooth$ and $\distcal$.

Our construction of $\atb$ is quadratically related to $\smooth$ and $\distcal$. As mentioned earlier, we prove this result by showing that $\atbone$ (\Cref{def: atb1}) gives a constant-factor approximation for $\smooth$ and $\distcal$.
Here we explain the intuition behind this analysis.

Our analysis is divided into the following two results, showing the upper and lower bounds on $\atbone$ separately:
\begin{align}
\text{\Cref{lm:atb-distcal}:  }\quad &
\atbone(\vreport,\vstate)  \le 3\, \distcal(\vreport,\vstate) \label{eq:tech-2-1}\\
\text{\Cref{thm:smooth-atb}: }\quad & \smooth(\vreport,\vstate)  \le \frac 32\, \atbone(\vreport,\vstate) \label{eq:tech-2-2}.
\end{align}
The desired constant-factor approximation (\Cref{thm: intro atb1 constant approx}) then follows from the previous result that $\smooth$ and $\distcal$ are themselves constant-factor approximations of each other (\Cref{prop:dist-smooth}) \citep{utc}.

While neither inequality is straightforward to prove, the relatively more technically involved and, perhaps, more surprising direction is the latter inequality \eqref{eq:tech-2-2} showing that $\smooth$ can be upper-bounded by $\atbone$ up to a constant factor.
Indeed, the intuition behind the previous notion of \emph{interval calibration error} $\intc$ \citep{utc} is that having too few bins tends to underestimate $\smooth$, and if the calibration error is much smaller than the average bin width, we should increase the number of bins to faithfully capture $\smooth$.\footnote{Consequently, the number of bins used to define $\intc(\vreport,\vstate)$ depends on both $\vreport$ and $\vstate$. In UBSE, the binning scheme can only depend on $\vreport$ in order for our truthfulness analysis to hold.}
The reasoning is that having fewer bins makes more predictions fall into the same bin, among which the positive and negative biases $r_t - y_t$ cancel out, thus more likely to cause underestimation.
For example, having only one bin gives the following UBSE:
\[
\CAL(\vreport,\vstate) = \left(\frac 1T\sum_{t = 1}^T(r_t - y_t)\right)^2,
\]
which clearly underestimates $\smooth$ (it can be zero even when $\vreport$ is mis-calibrated, in which case $\smooth(\vreport,\vstate)$ is always positive).
Therefore, based on this previous intuition, it is somewhat surprising that having just two bins  suffices to establish \eqref{eq:tech-2-2}.

Proving \eqref{eq:tech-2-2} is equivalent to showing that for any $1$-Lipschitz weight function $w:[0,1]\to [-1,1]$,
\begin{equation}
\label{eq:threshold-goal}
\frac{1}{T}\sum_{t\in [T]}w(\report_t)\cdot(\report_t - \state_t) \le \frac 32 \atbone(\vreport,\vstate).
\end{equation}
This equivalence follows from the definition of $\smooth$: it is the supremum of the left-hand side over all $1$-Lipschitz $w:[0,1]\to [-1,1]$ (\Cref{def:smooth}).

\begin{figure}
    \centering
    \includegraphics[scale=0.5]{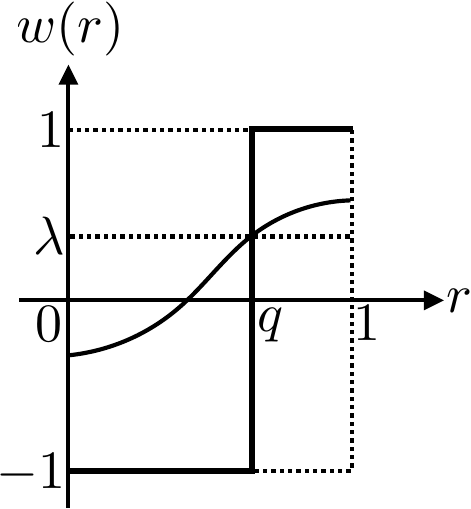}
    \caption{Writing $w$ as a convex combination of threshold functions.}
    \label{fig:threshold}
\end{figure}

To illustrate our proof idea, let us first assume that the weight function $w$ is not only Lipschitz, but also monotonically increasing and differentiable (represented by the curve in \Cref{fig:threshold}). The key observation is that we can write $w$ as a convex combination of threshold functions as follows.
Take a random threshold $\lambda$ uniformly distributed from $[-1,1]$ and consider the threshold function $w_\lambda(r) := \sign(w(r) - \lambda)$ (represented by the bold step function in \Cref{fig:threshold}). That is, $w_\lambda(r) = 1$ if $w(r) \ge \lambda$, and $w_\lambda(r) = -1$ if $w(r) < \lambda$. The following key identity expresses $w$ as a convex combination of the threshold functions $w_\lambda$:
\begin{equation}
\label{eq:threshold-0}
w(r) = \E_{\lambda\sim \mathrm{Unif}([-1,1])}[w_\lambda(r)]\quad \text{for every $r\in [0,1]$}.
\end{equation}
Now for a fixed threshold $\lambda\in [-1,1]$, let $q:=w^{(-1)}(\lambda)\in [0,1]$ be the corresponding threshold on the $r$-axis, where $w^{(-1)}$ is the inverse of $w$ (see \Cref{fig:threshold}). In the boundary cases when $\lambda > w(1)$, we choose $q = 1$, and similarly, when $\lambda < w(0)$ we choose $q = 0$. This ensures\footnote{One tiny caveat which we ignore here is that when $\lambda > w(1)$ and thus $q = 1$, this identity does not hold at one point: $r = 1$.}
\begin{equation}
\label{eq:threshold-0-1}
w_\lambda(r) = \sign(r - q) \quad \text{for every $r\in [0,1]$}.
\end{equation}

Let $Q$ be the distribution of the resulting $q$ from~$\lambda\sim \mathrm{Unif}([-1,1])$. By \eqref{eq:threshold-0} and \eqref{eq:threshold-0-1}, we can rewrite the left-hand side of \eqref{eq:threshold-goal} as
\begin{equation}
\label{eq:threshold-1}
\frac{1}{T}\sum_{t\in [T]}w(\report_t)\cdot(\report_t - \state_t) = \E_{q\sim Q}\left[\frac{1}{T}\sum_{t\in [T]}\sign(r_t - q)\cdot(\report_t - \state_t)\right].
\end{equation}
For each fixed choice of $q$, it is straightforward to show that the quantity inside the expectation in \eqref{eq:threshold-1} is upper-bounded by $\atbone$ at the same fixed bin threshold $q$ (\Cref{def: atb1}). However, the random variable $q$ is distributed differently in the two cases. It is drawn from the distribution $Q$ in \eqref{eq:threshold-1}, whereas it is uniformly distributed over $[0,1]$ in the definition of $\atbone$.

What remains is to relate the two distributions: $Q$ and $\mathrm{Unif}([0,1])$. Recall that $q\sim Q$ is obtained as $q = w^{(-1)}(\lambda)$ for uniformly distributed $\lambda\in [-1,1]$. It follows that the probability density function (PDF) of $q\sim Q$ is exactly the PDF of $\lambda$ (which is $1/2$ everywhere in $[-1,1]$) times the derivative $\nabla w(q)$, except at the boundaries $q = 0,1$. Since $w$ is $1$-Lipschitz, we have $\nabla w(q) \le 1$, and thus the PDF of $q\sim Q$ is at most $1/2$ everywhere in the open interval $(0,1)$. This is sufficient to bound the expectation over $q\sim Q$ in \eqref{eq:threshold-1} by the expectation over $\mathrm{Unif}([0,1])$ in the definition of $\atbone$ (\Cref{def: atbone}). The boundary cases of $q = 0,1$ need to be handled separately, but that turns out to be relatively straightforward.

To fully prove \eqref{eq:tech-2-2}, we need to remove the monotonicity and differentiability assumptions on $w$, which is achieved by our formal proof in \Cref{sec:atbone}. Roughly speaking, without monotonicity, the convex combination of the threshold functions that expresses $w$ might have negative coefficients (so it is a linear combination rather than a convex combination), but the absolute values of the coefficients can still be controlled using the Lipschitzness of $w$. The differentiability assumption can be removed by focusing on the finite set $\{r_1,\ldots,r_T\}$ rather than the full domain $[0,1]$ of $w$.

\subsection{Truthfulness and Monotonicity}
\label{sec:monotonicity}
In this subsection, we include an informal discussion about why it may appear challenging to construct a perfectly truthful calibration measure. This may give a partial explanation for why truthful calibration measures have not appeared before our work, despite the variety of calibration measures introduced in the literature.

At first glance, the existence of a perfectly truthful calibration measure may appear counterintuitive, which arises from the (seemingly) conflicting monotonicity of a truthful error metric and a sound and complete calibration measure. Previous theories on truthful error metrics (a.k.a.\ proper losses) show that truthfulness is closely tied to convexity and monotonicity. In contrast, the set of calibrated predictors is not convex, and a complete and sound calibration measure should not monotonically decrease as we move the predictions closer to the truth.


More concretely, \citet{lambert2011elicitation} characterizes a truthful error metric to be monotonically decreasing when a prediction moves closer to the truth. Formally, a truthful error of a report must be higher than any convex combination of the report with the state, shown in \Cref{fig:intro order sensitive}.
Calibration, however, specifies a conflicting monotonicity where all calibrated predictors are ranked lowest. As an example, suppose the $T$ realized states are $(0, 1, 0, 1, \dots)$ alternating between $0$ and $1$, with an empirical frequency of $50\%$. An uninformative predictor who always predicts $50\%$ should achieve a vanishing calibration error as well as a perfect predictor whose outputs deterministically match the states. Now consider the uniform interpolation between $50\%$ and the ground truth, an alternating prediction sequence of $(25\%, 75\%, 25\%, \dots)$. Intuitively, this interpolation seems to move ``closer'' to the perfect predictor. By monotonicity, the error of the interpolation should be upper-bounded by the vanishing error of reporting $50\%$. Yet this interpolation is very miscalibrated, and soundness requires it to receive a high calibration error. 

\begin{figure}
    \centering
    \begin{tikzpicture}[scale=0.3]

   \draw[thick] (0,0) circle (7.1);


  
    


  \draw[dashed, thick] (-5,5) -- (0,0);

    \draw[thick, ->] (-1,4) -- (2,1) node  [right]{\small increasing loss};

  \filldraw[black] (-5,5) circle (0.3) node [above left]{\small{realized state}};
  \draw[black, thick, fill=white] (-2.5,2.5) circle (0.3);
  \filldraw[black] (0,0) circle (0.3) node [below]{\small report};

\end{tikzpicture}
    \caption{The order sensitivity of a truthful error metric. The large circle is an abstraction of the probabilistic space, with a realized state on a corner of the space. The reported prediction lies in the interior of the space. Fixing the realized state, the truthful error, as a function of the prediction, is increasing along the convex combination from the realized state to the reported prediction. For one binary state prediction, fixing the realized state, a truthful error is monotone in the distance between the reported prediction and the state. 
    }
    \label{fig:intro order sensitive}
\end{figure}
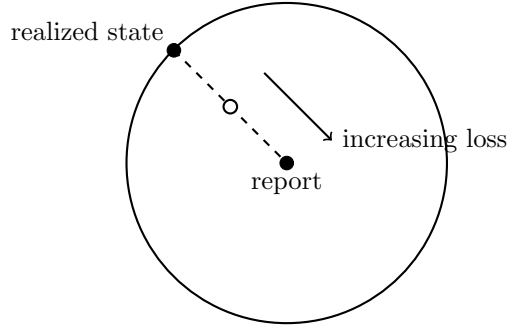

Perhaps due to the reasons above, previous \citep{haghtalab2024truthfulness} and concurrent \citep{qiao2025truthfulness} work focuses on achieving the weaker \textit{approximate} truthfulness in the more challenging online setting. 
We show that there exists a \textit{perfectly} truthful calibration error in the offline setting. The reason behind the counterintuitive possibility of perfect truthfulness is that when the states are drawn independently across individuals (i.e., a product distribution), the monotonicity of a truthful calibration measure is not violated, as the interpolation of two non-trivial product distributions is no longer a product distribution. Thus, the intuition of ``moving closer'' in the example above does not hold rigorously true. 
Our truthful construction based on variance additivity indeed crucially uses this independence / product distribution structure. 


\subsection{Related Work}

\paragraph{Truthful Calibration Errors.}
Previous work \citep{haghtalab2024truthfulness, qiao2025truthfulness} on approximate truthful calibration errors are closest to our paper. They design multiplicatively truthful calibration error in the sequential prediction problem. In the sequential prediction setting, a sequence of $T$ potentially correlated states is drawn from a distribution. At each period, the predictor predicts, and one state is revealed. Our work studies a different batch setting where all $T$ states are independently drawn and revealed simultaneously after all predictions. An error metric is approximately truthful if predicting the true conditional probability of the next state is a constant approximation of the optimal strategy. \citet{haghtalab2024truthfulness} shows that subsampled smooth calibration error is multiplicatively truthful for the sequential prediction setting, implying the smooth calibration error is multiplicatively truthful for the batch setting. \citet{qiao2025truthfulness} shows that, in the sequential setting, there does not exist a perfectly truthful calibration error that upperbounds the worst-case external regret for decision-makers.  The impossibility in sequential setting does not apply to our problem. It also remains open whether there exists a perfectly truthful calibration error metric for the sequential setting while satisfying other completeness and soundness properties.

\paragraph{Calibration Measures.} \citet{foster1997calibrated} first proposes the Expected Calibration Error (ECE). The binned ECE serves as a widely-used empirical proxy of ECE \citep{guo2017calibration, minderer2021revisiting}. \citet{kleinberg2023u} observes that, if predictions are used for downstream decision-making, $\ECE$ upperbounds the swap regret of any downstream decision-maker. Following the decision-making purpose of calibration, \citet{hu2024predict} proposes Calibration Decision Loss (CDL), the worst-case swap regret of any normalized downstream decision task, and shows CDL is quantitatively different from ECE. \citet{proper-cal} introduce the notion of proper calibration as a key ingredient for designing improved algorithms for omniprediction \citep{omni,loss-oi}.

\citet{utc} introduced the distance to calibration. In their framework, a calibration error is consistent if it is polynomially related to the distance to calibration.
They showed that the smooth calibration error \citep{smooth} and the Laplace kernel calibration error \citep{kernel} are both consistent, and introduced a binning-based consistent calibration error: the interval calibration error.


\paragraph{Proper Scoring Rules (a.k.a. truthful losses).} Initiated by \citet{Mcc-56, sav-71}, extensive work focused on the characterization of proper scoring rules, the class of truthful loss functions. \citet{lambert2011elicitation} characterizes elicitable statistics of a distribution, for example, the mean of a distribution, not the variance of a distribution. \citet{winkler1996scoring} provides proper scoring rules for the confidence interval, and \citet{frongillo2014general} provides a characterization of proper scoring rules for eliciting linear properties. \citet{li2022optimization} gives computational results of proper scoring rules.  

\subsection{Paper Organization}
The rest of the paper is organized as follows.
\Cref{sec: prelim} establishes the basic setup, including the definitions of existing calibration errors (\Cref{sec: prelim existing error}), completeness and soundness (\Cref{sec: prelim c and s}), the validity of calibration errors via calibration tests (\Cref{sec: prelim calibration test}), and the truthfulness of calibration errors (\Cref{sec: prelim truthful}). In \Cref{sec:ubse}, we introduce the Unnormalized Binned Squared Errors (UBSEs), a general family of truthful binning-based error metrics. In \Cref{sec:atb}, we introduce our proposed calibration error, the Averaged Two-Bin calibration error ($\atb$)  as a special case of UBSE and prove its truthfulness, continuity, sample efficiency, and computationally efficiency. In \Cref{sec:atbone}, we prove the quadratic relationship between $\atb$ and the existing calibration errors $\smooth,\distcal$ by showing that $\atbone$ is a constant-factor approximation of $\smooth$ and $\distcal$.
In \Cref{sec:test}, we show that $\atb$ is optimally valid for $\smooth$ and $\distcal$, implying a linear-time calibration tester for $\smooth$ and $\distcal$.  

\section{Preliminaries}
\label{sec: prelim}
Throughout the paper, we use $D$ to denote a joint distribution of $(x,y)$ pairs, where $x\in X$ represents an individual in a domain $X$, and $y\in \{0,1\}$ is the corresponding state (a.k.a.\ outcome or label). A predictor $r:X\to [0,1]$ reports a prediction $r(x)\in [0,1]$ for each individual $x\in X$.


We present useful definitions and preliminary theorems for our paper. \Cref{sec: prelim existing error} introduces existing calibration measures. \Cref{sec: prelim c and s} defines the completeness and soundness of a calibration measure. \Cref{sec: prelim truthful} formalizes truthfulness of an error measure.  \Cref{sec: prelim calibration test} introduces calibration test, preparing for the result on linear-time calibration tester.

\subsection{Calibration}
\label{sec: prelim existing error}



We present the formal definitions of a few important calibration error metrics in the literature. We start with the definition of calibration:

\begin{definition}[Calibration]
\label[definition]{def:calibration}
    A predictor $\report:X\to [0,1]$ is calibrated on an underlying distribution $D$ of $(x,y)\in X\times \{0,1\}$ if $\E_{D}[y|r(x)] = r(x)$ holds almost surely. 
\end{definition}

An important property of the definition of calibration is that it only depends on the distribution of the prediction-state pair $(r(x),y)\in [0,1]\times \{0,1\}$. That is, we can determine whether a predictor $r$ is calibrated on a distribution $D$ just based on the distribution of $(r(x),y)$, without having to know the full joint distribution of $(x,r(x),y)$. Thus, using a random variable $v$ to represent the prediction value $r(x)$, we can define calibration simply given a distribution $J$ of $(v,y)\in [0,1]\times \{0,1\}$:
\begin{definition}[Calibration of prediction-state distributions]
\label[definition]{def:calibration-2}
    We say a distribution $J$ of $(v,y)\in [0,1]\times \{0,1\}$ is calibrated if $\E_J[y|v] = v$ holds almost surely.
\end{definition}

For a distribution $D$ of $(x,y)\in X\times \{0,1\}$ and a predictor $r:X\to [0,1]$, we use $J_{D,r}$ to denote the joint distribution of $(r(x),y)$. With that, $r$ is calibrated on $D$ if and only if $J_{D,r}$ is calibrated as in \Cref{def:calibration-2}.

A calibration measure $\CAL_D(r)\in \R_{\ge 0}$ evaluates the deviation of a predictor $r$ from being perfectly calibrated on a distribution $D$.
Naturally, we define a calibration measure $\CAL(J)$ first for general prediction-state distributions $J$ of $(v,y)\in [0,1]\times \{0,1\}$, and then define 
\[
\CAL_D(r) : = \CAL(J_{D,r}).
\]

%
The most well-known calibration measure is the \emph{expected calibration error (ECE)}:
\begin{definition}[Expected Calibration Error (ECE),  \citealt{foster1997calibrated}]
\label[definition]{def:ece}
    Let $J$ be a distribution of $(v,y)\in [0,1]\times \{0,1\}$, and let random variable $\widehat v:= \E_J[y|v]$ be the conditional expectation of the state $y$ given the prediction value $v$. The expected calibration error (ECE) is defined as
    \[
    \ECE(J):= \E|v - \widehat v|.
    \]
    Correspondingly, for a distribution $D$ of $(x,y)\sim X\times \{0,1\}$ and a predictor $r:X\to [0,1]$, defining $\hat r(x):= \E_D[y|r(x)]$, we have
    \[
    \ECE_D(r):= \ECE(J_{D,r}) = \E_{D}|r(x) - \widehat r(x)|.
    \]
More generally, for every $\alpha \ge 1$, we define $\ell_\alpha$-ECE as follows:
\[
    \ell_\alpha\textup-\ECE(J):= \E[|v - \widehat v|^\alpha], \quad \ell_\alpha\textup-\ECE_D(r):= \E_D[|r(x) - \hat r(x)|^\alpha].
\]
\end{definition}
A downside of the ECE is its discontinuity: slight changes in the predictions $r(x)$ can cause significant changes to the ECE value. 
This motivated \cite{utc} to introduce a continuous calibration error metric, termed the \emph{distance to calibration}. It measures the earthmover distance from the prediction-state distribution $(v,y)$ to a calibrated distribution $(u,y)$.
\begin{definition}[(Lower) Distance to Calibration, \citealt{utc}]
\label[definition]{def:distcal}
Let $J$ be a distribution of $(v,y)\in [0,1]\times\{0,1\}$. Consider a joint distribution (i.e.\ coupling) $\Pi$ of $(u,v,y)\in [0,1]\times [0,1]\times \{0,1\}$, where $(v,y)$ is distributed according to $J$, and the distribution of $(u,y)$ is calibrated as in \Cref{def:calibration-2}. The (lower) distance to calibration is defined as the following infimum over all such couplings $\Pi$:
\[
\distcal(J):= \inf_{\Pi}\E_\Pi|u - v|.
\]
Correspondingly, given a distribution $D$ of $(x,y)\sim X\times \{0,1\}$ and a predictor $r:X\to [0,1]$, we define $\distcal_D(r):= \distcal(J_{D,r})$.
\end{definition}

One might imagine a different definition of the distance to calibration as the minimum $L_1$ distance $\E_D|r(x) - r'(x)|$ from the given predictor $r$ to a calibrated predictor $r'$. 
Indeed, this notion (denoted $\Ddistcal$) is the first definition of the distance to calibration introduced by \cite{utc}.
However, as shown by \cite{utc}, this definition is different from the $\distcal$ in \Cref{def:distcal} and has the disadvantage of depending on the full joint distribution of $(x,r(x),y)$, not just the prediction-state distribution of $(r(x),y)$. To address this disadvantage, \cite{utc} introduced the $\distcal$ in \Cref{def:distcal} and termed it the \emph{lower} distance to calibration. They also showed that the two definitions are quadratically related:
\[
\frac 1{16}\,\Ddistcal_D(r)^2 \le \distcal_D(r) \le \Ddistcal_D(r).
\]
We will focus on the lower distance to calibration in \Cref{def:distcal} throughout the paper and will often drop the word ``lower'' for brevity.


Another important continuous calibration measure is the \emph{smooth calibration error} introduced by \citet{smooth} (originally termed \emph{weak calibration}). As shown by \citet{utc}, the smooth calibration error $\smooth$ gives a constant factor approximation to $\distcal$ (see \Cref{prop:dist-smooth} below).

\begin{definition}[Smooth Calibration Error \citep{smooth}]
\label[definition]{def:smooth}
    Let $W_1$ be the family of $1$-Lipschitz functions $w:[0,1]\to [-1,1]$. For any distribution $J$ of $(v,y)\in [0,1]\times\{0,1\}$, the smooth calibration error is defined as
    \begin{equation}
    \label{eq:smooth}
    \smooth(J):= \sup_{w\in W_1}\E_J[(v - y)w(v)].
    \end{equation}
    Correspondingly, for a distribution $D$ of $(x,y)\sim X\times \{0,1\}$ and a predictor $r:X\to [0,1]$, we have
    \[
    \smooth_D(r):= \smooth(J_{D,r}) = \sup_{w\in W_1}\E_J[(r(x) - y)w(r(x))].
    \]
\end{definition}

Without the Lipschitzness constraint on $w$, the smooth calibration error would become the ECE (\Cref{def:ece}), where the supremum in \eqref{eq:smooth} is achieved by
\begin{align*}
    w(v) = \begin{cases}
     1,    &  \text{if }\hat v > v;\\
      -1,   & \text{otherwise.}
    \end{cases}
\end{align*}

The following proposition shows that $\distcal$ and $\smooth$ are constant factor approximations of each other:
\begin{proposition}[\citealt{utc}]
\label[proposition]{prop:dist-smooth}
    For any distribution $J$ of $(v,y)\in [0,1]\times \{0,1\}$,
    \begin{equation*}
       \frac{1}{2}\,\distcal(J) \leq \smooth(J) \leq 2\,\distcal(J).
    \end{equation*}
\end{proposition}


\subsection{Completeness and Soundness over Samples}
\label{sec: prelim c and s}
A basic property shared by all the calibration measures in \Cref{sec: prelim existing error} is that they are all minimized when the predictor is calibrated, with the minimum value being zero:
\begin{claim}
\label{claim:well-defined}
For $\CAL\in \{\ECE,\distcal,\smooth\}$, we have $\CAL(J) \ge 0$ for any distribution $J$ of $(v,y)\in [0,1]\times \{0,1\}$. Moreover,
\[
\CAL(J) = 0 \Longleftrightarrow \text{$J$ is calibrated (\Cref{def:calibration-2})}.
\]
\end{claim}

The claim above tells us that we can determine whether a predictor $r$ is calibrated on a distribution $D$ by checking whether the calibration error $\CAL_D(r)$ is zero. In practice, however, we rarely get access to the full distribution $D$ and can only compute the calibration error on an i.i.d.\ sample drawn from $D$. As we shall see, the property in \Cref{claim:well-defined} does not necessarily generalize to this sample-access scenario, even in the limit of infinitely large sample sizes. To formalize this intuition, we define completeness and soundness below.

Completeness requires that if a predictor is calibrated, the calibration error should vanish as the sample size increases:

\begin{definition}[Completeness]
\label[definition]{def:consistent}
    We say a calibration measure $\CAL$ is complete if the following holds. Let $J$ be an arbitrary distribution of prediction-state pairs $(v,y)\in [0,1]\times \{0,1\}$ and assume that $J$ is calibrated (see \Cref{def:calibration-2}). Let $S = \{(v_t,y_t)\}_{t\in T}$ be a sample of $T$ i.i.d.\ points drawn from $J$ (denoted $S\sim J^T$), and let $J_S$ denote the (empirical) uniform distribution over $S$. Then
    \[
    \lim_{T\to \infty}\E_{S\sim J^T}[\CAL(J_S)] = 0.
    \]
    %
    %
\end{definition}

Conversely, soundness requires that if a predictor is miscalibrated, the calibration error should not vanish as the sample size increases. 


\begin{definition}[Soundness]
\label[definition]{def:sound}
    We say a calibration measure $\CAL$ is sound if the following holds. Let $J$ be an arbitrary distribution of prediction-state pairs $(v,y)\in [0,1]\times\{0,1\}$ and assume that $J$ is mis-calibrated. (That is, $J$ does not satisfy \Cref{def:calibration-2}.) 
    Let $S = \{(v_t,y_t)\}_{t\in T}$ be a sample of $T$ i.i.d.\ points drawn from $J$, and let $J_S$ denote the (empirical) uniform distribution over $S$. Then
    \[
    \liminf_{T\to \infty}\E_{S\sim J^T}[\CAL(J_S))] > 0.
    \]
    %
    %
\end{definition}

It can be inferred from the work of \cite{utc} that $\smooth$ and $\distcal$ are both complete and sound. However,
while
    $\ECE$ satisfies \Cref{claim:well-defined}, it is not complete. To see this, consider the distribution $J$ of prediction-state pairs $(v,y)\in [0,1]\times \{0,1\}$, where $v$ is drawn uniformly from $[1/3,2/3]$, and conditioned on $v$, $y$ is drawn from the Bernoulli distribution with mean $v$. Clearly, $J$ is calibrated and $\ECE(J) = 0$. However, on a finite sample $S = \{(v_1,y_1),\ldots,(v_T,y_T)\}$ drawn i.i.d.\ from $T$, it holds almost surely that all the $v_t$'s are distinct, in which case $\ECE(J_S) \ge 1/3$ (see \Cref{example: intro nontruthful}).

Due to ECE's lack of completeness, in machine learning practice, the binned $\ECE$ ($\binECE$) is widely adopted as an empirical method for estimating $\ECE$ \citep{guo2017calibration, minderer2021revisiting}. While we do not need this notion to state our main results, we include its definition here for completeness:
\begin{definition}[Binned ECE]
\label[definition]{def:bin}
Let $J$ be a distribution of $(v,y)\in [0,1]\times \{0,1\}$. Given $\alpha \ge 1$ and a partition $\intvset = \{\intv_i\}_{i\in [k]}$ of the prediction space $[0, 1]$, the $\ell_\alpha$-binned $\ECE$ is defined as
    \[
    \ell_\alpha\textup-\binECE(J):= \sum_{i\in [k]}\Pr\nolimits_{J}[v\in I_i]\cdot \big|\E_{J}[v - y|v\in I_i]\big|^\alpha.
    \]
\end{definition}
We can estimate $\ell_\alpha\textup-\binECE(J)$ using a sample $S = \{(v_t,y_t)\}_{t\in T}$ of $T$ i.i.d.\ points drawn from $J$. Specifically, letting $J_S$ denote the (empirical) uniform distribution over $S$, we can use $\ell_\alpha\textup-\binECE(J_S)$ as a good estimate for $\ell_\alpha\textup-\binECE(J)$ when the sample size $T$ is sufficiently large relative to $k$ (the number of bins).
%
In practice, the number $k$ of bins can be selected according to the sample size $T$, e.g.\ $k = T^{\frac{1}{3}}$, to ensure soundness and completeness.



\begin{remark}[Comparison to \citealt{haghtalab2024truthfulness}]
Our definition of completeness follows the same idea as \citet{haghtalab2024truthfulness}, while our soundness is strictly stronger. There exists an error metric that is not reasonably sound but satisfies the completeness and soundness in \citet{haghtalab2024truthfulness}. 

The soundness definition in \citet{haghtalab2024truthfulness} requires that for any empirical distribution $\distribution_T$ over $T$ samples,
\begin{itemize}
    \item if $\report_t = 1-\state_t$ for all $t$, then $\lim_{T\to \infty}\CAL_T(\report) = \Omega(1)$; 
    \item if each state $\state\sim \bernoulli(\alpha)$ is independent and identical Bernoulli, then $\lim_{T\to \infty}\CAL_T(\report) = \Omega(1)$ for $\report \neq \beta$ being a non-truthful constant. 
\end{itemize}

We see that the error $\CAL = (\expect{}{\report} - \expect{}{\state})^2 + \expect{}{\ind{\report\in \{0, 1\}, \state \neq \report}}$ satisfies the requirements above. However, for predictions not in $\{0, 1\}$, the error metric only evaluates the unconditional bias in predictions, which is far from a calibration error metric. 
\end{remark}

\subsection{Truthfulness}
\label{sec: prelim truthful}

A truthful error metric incentivizes a strategizing predictor to report the true distribution to minimize expected error on a finite sample. 
\Cref{def: ex ante truthful} defines the ex-ante truthfulness where a predictor output is assumed to be function of the feature space.

\begin{definition}[Ex-Ante Truthfulness]
\label[definition]{def: ex ante truthful}
    We say a calibration measure $\CAL$ is ex-ante truthful if the following holds. Let $D$ be an arbitrary joint distribution of $(x,y)\in X\times\{0,1\}$ and let $p:X\to [0,1]$ be the ground-truth predictor $p(x) = \E_D[y|x]$. Let $S = \{(x_t,y_t)\}_{t\in [T]}$ be a sample of $T$ i.i.d.\ points drawn from $D$, and let $D_S$ denote the (empirical) uniform distribution over $S$. Then
    \[
    \E_S[\CAL_{D_S}(p)] \le \E_S[\CAL_{D_S}(r)] \quad \text{for every predictor $r:X\to [0,1]$}.
    \]

\end{definition}

In this paper, we study a strictly stronger notion: interim truthfulness. In the interim stage, the true distribution of $T$ samples are realized, and the predictor is allowed to deviate and report any prediction sequence, not necessarily a function of the feature space. 
We first extend our definition of calibration errors to this setting, where we evaluate the calibration error of a reported prediction sequence $\vreport = (r_1,\ldots,r_T)$ for the $T$ individuals w.r.t.\ a ground-truth probability sequence $\vpred = (p_1,\ldots,p_T)$. We will refer to both definition as truthfulness when it is clear from the context.


\begin{definition}[Induced calibration error on prediction sequences]
\label[definition]{def: induced error}
    Given a calibration measure $\CAL(J)$ defined on prediction-state distributions $J$ over $[0,1]\times \{0,1\}$, we define an induced calibration measure $\CAL(\vreport,\vpred)$ as follows, where $\vreport = (\report_1,\ldots,\report_T)\in [0,1]^T$ is a sequence of predictions and $\vpred = (\pred_1,\ldots,\pred_T)\in [0,1]^T$ is a sequence of ground-truth probabilities. Let $J_{\vreport,\vpred}$ be the distribution of $(r_t,y)\in [0,1]\times \{0,1\}$ where $t$ is drawn uniformly from $[T]$, and $y\in \{0,1\}$ is drawn from the Bernoulli distribution with mean $\pred_t$. 
    We define
    \[
    \CAL(\vreport,\vpred) := \CAL(J_{\vreport,\vpred}).
    \]
\end{definition}
For example, according to \Cref{def: induced error}, we can explicitly calculate $\ECE(\vreport,\vpred)$ and $\smooth(\vreport,\vpred)$ as follows. Recall that for $v\in \{r_1,\ldots,r_T\}$, we define
\begin{equation}
\label{eq:hat-v}
\hat v:= \E_{(v,y)\sim J_{\vreport,\vpred}}[y|v] = \frac{\sum_{t\in [T]}p_t\ind{r_t = v}}{\sum_{t\in [T]}\ind{r_t = v}}.
\end{equation}
We have
\begin{align*}
    \ECE(\vreport,\vpred) = \ECE(J_{\vreport,\vpred}) & = \expect{(v, \state)\sim J_{\vreport,\vpred}}{\big|v - \hat{v}\big|}\\
    &=  \frac{1}{T}\sum_{v}\sum_{t\in [T]}\ind{\report_t = v}\big| v - \hat v\big| \tag{$v$ ranges over all values that appear at least once in the set $\{r_1,\ldots r_t\}$} \\
    &=  \frac{1}{T}\sum_{v}\left| (v - \hat v)\sum_{t\in [T]}\ind{\report_t = v}\right| \\
    & =\frac 1T \sum_{v}\left|\sum_{t\in [T]} (r_t - p_t) \ind{r_t = v}\right|.\tag{by \eqref{eq:hat-v} and $v\ind{r_t = v} = r_t\ind{r_t = v}$}
\end{align*}
Similarly for $\smooth$:
\begin{align*}
\smooth(\vreport,\vpred) & = \sup_{w\in W_1} \frac 1T\sum_{t = 1}^T (r_t - p_t)w(r_t).
\tag{$W_1$ is the same as in \Cref{def:smooth}}
\end{align*}




We now define the notion of truthfulness for the calibration errors from \Cref{def: induced error} on length-$T$ sequences. We note that this definition is akin to the definition of properness in the literature of proper scoring rules \citep{Mcc-56, sav-71}.

\begin{definition}[Interim Truthfulness]
\label[definition]{def:truthful-sequence}
We say a calibration measure $\CAL$ is interim truthful if the following holds for any $T\in \Z_{> 0}$. Let $\vpred:= (\pred_1,\ldots,\pred_T)\in [0,1]^T$ be an arbitrary sequence of ground-truth predictions. Let $\vstate = (\state_1,\ldots,\state_t)$ denote the randomly realized states, where each $\state_t\in \{0,1\}$ is drawn independently from the Bernoulli distribution with mean $\pred_i$ (denoted $\vstate\sim \vpred$). Then
\[
\E_{\vstate\sim \vpred}[\CAL(\vpred,\vstate)] \le \E_{\vstate\sim \vpred}[\CAL(\vreport, \vstate)] \quad \text{for any $\vreport = (\report_1,\ldots,\report_T)\in [0,1]^T$}.
\]
\end{definition}

\begin{claim}
    [Interim truthfulness implies ex-ante truthfulness]
    Let $\CAL(\vreport,\vpred)$ be a calibration measure induced by $\CAL(J)$ (\Cref{def: induced error}). If $\CAL(\vreport,\vpred)$ is interim truthful, then $\CAL(J)$ is ex-ante truthful.
\end{claim}
\begin{proof}
    As in \Cref{def: ex ante truthful}, consider a sample $S = \{(x_1,y_1),\ldots,\allowbreak (x_T,y_T)\}$ of i.i.d.\ points from a distribution $D$ over $X\times \{0,1\}$, and let $r:X\to [0,1]$ be a predictor. Define $\vreport := (r(x_1),\ldots,r(x_T))$ and $\vstate:= (y_1,\ldots,y_T)$. Now $J_{D_S,r}$ and $J_{\vreport,\vstate}$ are both equal to the distribution of $(r(x_t),y_t)$ for uniform $t\in [T]$. Therefore,
    \begin{equation}
    \label{eq:sample-to-sequence-1}
    \CAL_{D_S}(r) = \CAL(J_{D_S,r}) = \CAL(J_{\vreport,\vstate}) = \CAL(\vreport,\vstate).
    \end{equation}

    As in \Cref{def: ex ante truthful}, define $p(x_t) := \E_D[y|x = x_t]\in [0,1]$ for $t = 1,\ldots,T$. Conditioned on $x_1,\ldots,x_T$, each $y_t$ is distributed independently from the Bernoulli distribution with mean $p(x_t)$. 
    That is, we have $\vstate\sim \vpred$ as in \Cref{def:truthful-sequence}, where $\vpred := (p(x_1),\ldots,p(x_T))$.
Therefore, by \eqref{eq:sample-to-sequence-1},
    \begin{align}
    \E_S[\CAL_{D_S}(r) | x_1,\ldots,x_T] & = \E_{\vstate\sim\vpred}[\CAL(\vreport,\vstate)], \label{eq:sample-to-sequence-2}\\
    \E_S[\CAL_{D_S}(p) | x_1,\ldots,x_T] & = \E_{\vstate\sim\vpred}[\CAL(\vpred,\vstate)]. \label{eq:sample-to-sequence-3}
    \end{align}
    Assuming interim truthfulness, we know that the quantity in \eqref{eq:sample-to-sequence-2} is no smaller than the quantity in \eqref{eq:sample-to-sequence-3}. Taking the expectation over $x_1,\ldots,x_T$ proves the desired ex-ante truthfulness.
\end{proof}


\subsection{Calibration Test and Validity}
\label{sec: prelim calibration test}

Completeness and soundness (\Cref{def:consistent,def:sound}) ensure that a calibration measure $\CAL$ is able to distinguish calibrated predictors from mis-calibrated ones, when the sample size $T$ is large enough. Intuitively, we should expect the distinguishing power to grow as a function of $T$. We characterize this quantitative dependence on $T$ below. We first define calibration tests that aim at accepting calibrated predictors while rejecting mis-calibrated ones, based on a sample of size $T$.

\begin{definition}[Calibration Test]
\label[definition]{def: calibration test}
    Consider the following calibration test using a calibration measure $\CAL$. Let $J$ be an arbitrary distribution of prediction-state pairs $(v,y)\in [0,1]\times \{0,1\}$. The test first draws $T$ i.i.d.\ points from $J$ to form a sample $S = \{(v_t,y_t)\}_{t\in [T]}$, and then computes the calibration error $\CAL(J_S)$ on the uniform distribution $J_S$ over $S$. The test outputs ``accept'' if the calibration error does not exceed a threshold $\beta$. That is, the acceptance probability of this test is
    \[
    \accp^\CAL(J;T, \beta) := \Pr\nolimits_{S\sim J^T}[\CAL(J_S) \le \beta].
    \]
\end{definition}


We define the validity of a calibration measure $\CAL$ given a reference calibration measure $\referr$ that is often chosen to be complete and sound. 
\begin{definition}[Validity]
\label[definition]{def: validity}
    Let $\{\gamma_T\}$ be an infinite sequence of real numbers indexed by $T = 1,2,\ldots$. We say a calibration measure $\CAL$ is $\{\gamma_T\}$-valid w.r.t.\ a reference calibration measure $\referr$ if there exist thresholds $\beta_1,\beta_2,\ldots \in \R$ such that
    \[
    \liminf_{T\to \infty} \left(\inf_{J\,:\,\text{calibrated}}\accp^\CAL(J;T,\beta_T) - \sup_{J\,:\,\referr(J) \ge \gamma_T}\accp^\CAL(J;T,\beta_T) \right) > 0.
    \]
    That is, there is a non-vanishing gap between the acceptance probability when $J$ is calibrated, and the acceptance probability when $J$ is mis-calibrated with error at least $\gamma_T$ in the reference measure $\referr$.
    \end{definition}

In the definition above, one should typically think of $\gamma_T$ as a decreasing function of $T$, which indicates the stronger distinguishing power as $T$ grows. Moreover, the faster $\gamma_T$ decreases, the stronger is the distinguishing power of a $\{\gamma_T\}$-valid calibration error for large $T$.

\section{Truthful Family: Unnormalized Binned Squared Errors}
\label{sec:ubse}
In this section, we present a general family of truthful error metrics, which we term \emph{unnormalized binned squared errors (UBSEs)}. As it will become clear, the error $\atb$ is a special case of UBSEs, so its truthfulness is an immediate consequence of the truthfulness of UBSEs.

\begin{definition}[Unnormalized Binned Squared Errors]
\label[definition]{def:ubse}
Consider an error metric $\CAL(\vreport,\vpred)$ taking as input a report vector $\vreport = (r_1,\ldots,r_T)\in [0,1]^T$ and a ground-truth vector $\vpred = (p_1,\ldots,p_T)\in [0,1]^T$. We say $\CAL$ is an \emph{unnormalized binned squared error} (UBSE) if it can be calculated as follows:
\begin{enumerate}
    \item Partition the indices $[T]$ into $k$ disjoint bins: $[T] = B_1\cup \cdots \cup B_k$. Importantly, we allow the partition (including the choice of $k$) to be randomized, and we allow it to depend on the report vector $\vreport$ (but not on $\vpred$).
    \item Compute the bias $\Delta_i$ in each bin $B_i$:
    \begin{equation}
    \label{eq:unnormalized}
    \Delta_i := \frac 1T\sum_{t\in B_i}(r_t - p_t).
    \end{equation}
    \item Output the sum of the squared biases: $\CAL(\vreport, \vpred) := \E_{\cB}[\sum_{i=1}^k\Delta_i^2]$, where the expectation is over the randomness of the partition $\cB = (B_1,\ldots,B_k)$.
\end{enumerate}
\end{definition}
The above definition is very similar to the definition of binned $\ECEtwo$, but there is a crucial difference. When defining binned $\ECEtwo$ for a fixed partition $\cB = (B_1,\ldots,B_k)$, the bias in each bin is first \emph{normalized by the bin size $|B_i|$}:
\[
\widetilde \Delta_i = \frac {1}{|B_i|}\sum_{t\in B_i}(r_t - p_t),
\]
and then squared and summed with \emph{weights} $|B_i|/T$:
\[
\ECEtwo(\vreport,\vpred) = \sum_{i = 1}^k\frac{|B_i|}{T}\widetilde \Delta_i^2 = \sum_{i = 1}^k \frac{1}{|B_i|T}\left(\sum_{t\in B_i}(r_t - p_t)\right)^2.
\]
In contrast, \Cref{def:ubse} takes the \emph{unweighted} sum of the \emph{unnormalized} squared biases $\Delta_i^2$:
\[
\CAL(\vreport,\vpred) = \E_\cB\left[\sum_{i = 1}^k\Delta_i^2\right] = \E_\cB\left[\sum_{i=1}^k \frac 1{T^2}\left(\sum_{t\in B_i}(r_t - p_t)\right)^2\right].
\]

\subsection{Interim Truthfulness}

UBSEs is interim truthful (whereas the binned $\ECEtwo$ is not, with the small difference above):
\begin{theorem}
\label{thm:ubse}
    Any UBSE error metric $\CAL$ is interim truthful (\Cref{def:truthful-sequence}).
\end{theorem}
In fact, we prove the a stronger result in \Cref{lem: error decomposition}, 
showing that the expected empirical UBSE decomposes into the UBSE on the true probabilities $\vpred$ plus a variance term independent of $\vreport$. 
\begin{lemma}[Error Decomposition]
\label[lemma]{lem: error decomposition}
Let $\CAL$ be an arbitrary UBSE. For any report sequence $\vreport = (r_1,\ldots,r_T)\in[0,1]^T$ and any ground-truth vector $\vpred = (p_1,\ldots,p_T)\in [0,1]^T$,
\[
\E_{\vstate \sim \vpred}[\CAL(\vreport,\vstate)] = \CAL(\vreport,\vpred)  + \frac 1{T^2}\sum_{t = 1}^T p_t(1 - p_t).
\]
Here $\vstate = (y_1,\ldots,y_T)\in \{0,1\}^T$ is drawn such that each $y_t$ independently follows the Bernoulli distribution with mean $p_t$ (as in \Cref{def:truthful-sequence}).
\end{lemma}
We first prove \Cref{thm:ubse} using \Cref{lem: error decomposition}, and then prove \Cref{lem: error decomposition}.
\begin{proof}[Proof of \Cref{thm:ubse}]
    For any $\vreport,\vpred\in [0,1]^T$, by \Cref{lem: error decomposition},
    \begin{equation}
    \label{eq:ubse-proof}
    \E_{\vstate \sim \vpred}[\CAL(\vreport,\vstate)] - \E_{\vstate \sim \vpred}[\CAL(\vpred,\vstate)] = \CAL(\vreport,\vpred) - \CAL(\vpred,\vpred).
    \end{equation}
    Clearly, we have $\CAL(\vreport,\vpred) \ge 0$ and $\CAL(\vpred,\vpred) = 0$. Therefore, the quantity in \eqref{eq:ubse-proof} is non-negative, which means that $\CAL$ is interim truthful.
\end{proof}
\begin{proof}[Proof of \Cref{lem: error decomposition}]
For a partition $\cB = (B_1,\ldots,B_k)$ of $[T]$ as in \Cref{def:ubse}, we define
\begin{align*}
\Delta_i &:= \frac 1T\sum_{t\in B_i}(r_t - y_t).\\
\widehat \Delta_i & := \frac 1T\sum_{t\in B_i}(r_t - p_t).
\end{align*}
We have
\begin{align}
\E_{\vstate\sim\vpred}[\CAL(\vreport,\vstate)] & = \E_{\vstate\sim\vpred}\left[\E_{\cB}\left[\sum_{i=1}^k\Delta_i^2\right]\right] = \E_{\cB}\left[\E_{\vstate\sim\vpred}\left[\sum_{i=1}^k\Delta_i^2\right]\right], \label{eq:decomposition-1}\\
\CAL(\vreport,\vpred) & = \E_{\cB}\left[\sum_{i=1}^k\widehat\Delta_i^2\right].\label{eq:decomposition-1-1}
\end{align}
In \eqref{eq:decomposition-1}, we used the fact that the distribution of $\cB$ depends only on $\vreport$ and not on $\vstate$. For the same reason, the two distributions of $\cB$ in \eqref{eq:decomposition-1} and $\eqref{eq:decomposition-1-1}$ are the same. Therefore, to prove the lemma, it suffices to show that for any fixed partition $\cB$, 
\begin{equation}
\label{eq:decomposition-goal}
\E_{\vstate\sim\vpred}\left[\sum_{i=1}^k\Delta_i^2\right] = \sum_{i=1}^k\widehat\Delta_i^2 + \frac 1{T^2}\sum_{t = 1}^T p_t(1 - p_t).
\end{equation}
For every $i = 1,\ldots,k$, we have
\begin{equation}
\label{eq:decomposition-2}
\E_{\vstate\sim \vpred}[\Delta_i^2] = \E_{\vstate\sim\vpred}[\Delta_i]^2 + \var{\vstate\sim\vpred}{\Delta_i},
\end{equation}
where
\begin{align*}
\E_{\vstate\sim \vpred}[\Delta_i] & = \widehat \Delta_i,\\
\var{\vstate\sim \vpred}{\Delta_i} & = \var{\vstate\sim \vpred}{\frac 1{T}\sum_{t\in B_i}(r_t - y_t)}\\
& = \frac 1{T^2}\var{\vstate\sim \vpred}{\sum_{t\in B_i}y_t}\\
& = \frac 1{T^2}\sum_{t\in B_i}\var{\vstate\sim \vpred}{y_t} \tag{the $y_t$'s are distributed independently}\\
& = \frac 1{T^2}\sum_{t\in B_i}p_t(1 - p_t).
\end{align*}
Plugging these into \eqref{eq:decomposition-2}, we have
\[
\E_{\vstate\sim \vpred}[\Delta_i^2] = \widehat \Delta_i^2 + \frac 1{T^2}\sum_{t\in B_i}p_t(1 - p_t).
\]
Summing up over $i = 1,\ldots,k$ proves \eqref{eq:decomposition-goal}.
\end{proof}

We remark that in addition to being truthful, UBSEs are also complete. This is because, by \Cref{lem: error decomposition}, the expected error of calibrated predictions is
\[
\frac 1{T^2}\sum_{t = 1}^T p_t(1 - p_t) \le \frac{1}{4T} = O(1/T),
\]
which vanishes as $T\to \infty$.

\begin{example}[Quantile-Binned $\ECEtwo$ is truthful]
As a special case of $\ubse$, the quantile-binned $\ECEtwo$  is truthful and complete. Choosing the number of bins properly as a growing function of $T$, it is also a sound calibration error. It is defined as follows:

    For any report sequence $\vreport = (\report_1, \dots, \report_T)$ and any vector of realized state $\vstate = (\state_1, \dots, \state_T)$,
    \begin{itemize}
        \item sort the predictions in increasing order with $\vreport_1\leq \dots\leq \vreport_T$, break ties uniformly at random. 
        \item Partition predictions into $k = T^{\sfrac{1}{3}}$ bins by quantile. Each bin has $\frac{T}{k}$ predictions. 
        \item Given the partition above, output the Unnormalized Binned Squared Error $\CAL(\vreport, \vstate)$. 
    \end{itemize}
\end{example}

\subsection{Strict Ex-Ante Truthfulness}

In the ex-ante stage before the ground-truth probabilities for each sample are drawn, $\ubse$ is strictly truthful, i.e., the unique minimizer to the expected error is the ground truth predictor $p(x): = \E[y|x]$.

\begin{theorem}
    \label{thm:strict}
    Let $X$ be an arbitrary non-empty domain and let $D$ be an arbitrary distribution of $(x,y)\in X\times \{0,1\}$. Let $p:X\to [0,1]$ be the ground-truth predictor $p(x): = \E[y|x]$ and let $r:X\to [0,1]$ be an arbitrary predictor.  Let $\CAL$ be an arbitrary UBSE and $T$ be an arbitrary positive integer. For a sample $S = ((x_t,y_t))_{t\in [T]}$ of $T$ i.i.d.\ examples drawn from $D$, let $D_S$ denote the uniform distribution on $S$. Suppose
    \[
    \E_S[\CAL_{D_S}(r)] \le \E_S[\CAL_{D_S}(p)].
    \]
    Then $r(x) = p(x)$ holds almost surely over $(x,y)\sim 
    D$.
\end{theorem}
We prove \Cref{thm:strict} using the following lemma. We are able to prove a stronger version of this lemma for the special case of $\atb$ in \Cref{lm:atb-brier}.
\begin{lemma}
\label[lemma]{lm:ubse-brier}
    Let $J$ be an arbitrary distribution of $(v,y)\in [0,1]\times \{0,1\}$, and define random variable $\hat v := \E_J[y|v]$ as a function of $v$. Let $\CAL$ be an arbitrary UBSE. For $T$ i.i.d.\ examples $(\hat v_1,v_1,y_1),\ldots,\allowbreak (\hat v_T,v_T,y_T)$, defining $\hat {\bm v}:= (\hat v_1,\ldots,\hat v_T), \bm v:= (v_1,\ldots,v_T)$ and $\vstate = (y_1,\ldots,y_T)$, we have
\begin{align}
    \E[\CAL(\hat {\bm v},\vstate)] & = \frac 1T \E_J[(\hat v - y)^2],\label{eq:ubse-brier-1}\\
    \E[\CAL(\bm v,\vstate)] & \ge \E[\CAL(\hat {\bm v},\vstate)] + \frac 1{T^2} \E_J [(\hat v - v)^2].\label{eq:ubse-brier-2}
\end{align}
\end{lemma}
\begin{proof}
By the definition of $\hat v:= \E_J[y|v]$, the distribution of $(\hat v,y)$ is calibrated. Therefore, $\E[y_t|\hat v_t] = \hat v_t$ for every $t = 1,\ldots,T$.
    By \Cref{lem: error decomposition}, we have
\[
    \E[\CAL(\hat {\bm v},\vstate)] = \frac 1{T^2}\E\left[\sum_{t = 1}^T \hat v_t(1 - \hat v_t)\right] = \frac 1T \E[\hat v(1 - \hat v)] = \frac 1T\E[(\hat v - y)^2],
\]
where we use the fact that $\hat v_1,\ldots,\hat v_T$ are i.i.d.\ random variables. This completes the proof of \eqref{eq:ubse-brier-1}.

By \Cref{lem: error decomposition} again, we have
\begin{equation}
\label{eq:cal-diff}
\E[\CAL(\bm v,\vstate)] =  \E[\CAL(\bm v, \hat{\bm v})] + \frac 1{T^2}\E\left[\sum_{t = 1}^T \hat v_t(1 - \hat v_t)\right] = \E[\CAL(\bm v, \hat{\bm v})] + \E[\CAL(\hat {\bm v},\vstate)] .
\end{equation}
Since $\CAL$ is a UBSE (\Cref{def:ubse}) and there are at most $T$ non-empty bins, by the Cauchy-Schwarz inequality, we have
\[
\CAL(\bm v, \hat{\bm v}) \ge \frac 1{T^3}\left(\sum_{t = 1}^T(v_t - \hat v_t)\right)^2.
\]
Taking expectation, we get
\begin{align*}
\E[\CAL(\bm v, \hat{\bm v})] & \ge \frac 1{T^3}\left(T \E[(v - \hat v)^2] + T(T - 1)\E[v - \hat v]^2\right)\tag{because $(v_1,\hat v_1),\ldots,(v_T,\hat v_T)$ are i.i.d.}
\\
& \ge \frac 1{T^2}\E[(v - \hat v)^2].
\end{align*}
Plugging this into \eqref{eq:cal-diff} proves \eqref{eq:ubse-brier-2}.
\end{proof}
We are now ready to prove \Cref{thm:strict}.
\begin{proof}[Proof of \Cref{thm:strict}]
    Define predictor $\hat r:X\to [0,1]$ such that $\hat r(x) = \E_D[y|r(x)]$. By \Cref{lm:ubse-brier},
    \begin{align*}
    \E_S[\CAL_{D_S}(p)] &= \frac 1T \E_D[(p(x) - y)^2],\\
    \E_S[\CAL_{D_S}(r)] &\ge \frac 1T \E_D[(\hat r(x)- y)^2] + \frac 1{T^2} \E_D[(r(x) - \hat r(x))^2].
    \end{align*}
Since $\E[y|x] = p(x)$, we have
\begin{align*}
\E_D[(\hat r(x)- y)^2]  & =  \E[(\hat r(x) - p(x))^2] + \E[(p(x) - y)^2] + 2\E[(\hat r(x) - p(x))(p(x) - y)]\\
& = \E[(\hat r(x) - p(x))^2] + \E[(p(x) - y)^2].
\end{align*}
Therefore,
\begin{align*}
0 & \ge \E_S[\CAL_{D_S}(r)] - \E_S[\CAL_{D_S}(p)]\\
& \ge \frac 1T(\E[(\hat r(x) - y)^2] - \E[(p(x) - y)^2]) + \frac 1{T^2} \E[(r(x) - \hat r(x))^2]\\
& = \frac 1T \E[(\hat r(x) -p(x))^2] + \frac 1{T^2} \E[(r(x) - \hat r(x))^2].
\end{align*}
This implies that $r(x) = \hat r(x) = p(x)$ almost surely.
\end{proof}


\section{Calibration Errors with Two Bins}
\label{sec:atb}
In this section, we formally define our calibration measure: the \emph{averaged two-bin calibration error} (ATB). 
We show that $\atb$ satisfies the following properties in the literature: completeness and soundness, 
truthfulness, continuity, sample complexity for estimation, and computational efficiency.
Our proof of the completeness and soundness relies heavily on the quadratic relationship between $\atb$ and its $\ell_1$ variant ($\atbone$). 
We will show that $\atbone$ linearly approximates existing calibration measures, implying the completeness and soundness of both $\atbone$ and $\atb$. 
\begin{definition}
\label[definition]{def: atbone}
For any distribution $J$ of prediction-state pairs $(v,y)\in [0,1]\times \{0,1\}$, we define the \emph{averaged two-bin calibration error (ATB)} and its $\ell_1$ variant as follows:
    \begin{align*}
        \atb(J) & = \E_{q\sim \mathrm{Unif}([0,1])}\Big[\Big(\E_J\Big[\big(v - y\big)\ind{v < q}\Big]\Big)^2 + \Big(\E_J\Big[\big(v - y\big)\ind{v \ge q}\Big]\Big)^2 \Big],\\
    \atbone(J) & = \E_{q\sim \mathrm{Unif}([0,1])}\Big[\Big|\E_J\Big[\big(v - y\big)\ind{v < q}\Big]\Big| + \Big|\E_J\Big[\big(v - y\big)\ind{v \ge q}\Big]\Big| \Big].
    \end{align*}
%
Correspondingly, for any prediction sequence $\vreport\in [0,1]^T$ and ground-truth sequence $\vpred\in [0,1]^T$,
    \begin{align}
    \atb(\vreport,\vpred) & = \E_{q\sim \mathrm{Unif}([0,1])}\left[\frac 1{T^2} \left(\left(\sum_{t: \report_t < q} (\report_t - \pred_t)\right)^2 + \left(\sum_{t: \report_t \ge q}(\report_t - \pred_t)\right)^2\right)\right],\label{eq:atb-sequence}\\
        \atbone(\vreport,\vpred) & = \E_{q\sim \mathrm{Unif}([0,1])}\left[\frac 1T \left(\left|\sum_{t: \report_t < q} (\report_t - \pred_t)\right| + \left|\sum_{t: \report_t \ge q}(\report_t - \pred_t)\right|\right)\right].\notag
    \end{align}
\end{definition}
To prepare for the proof, $\atbone$ is quadratically related to $\atb$ by Jensen's inequality.
\begin{lemma} 
\label[lemma]{lem: atb 1-atb}
For any distribution $J$ of prediction-state pairs $(v,y)\in [0,1]\times \{0,1\}$, 
\begin{equation*}
    \frac 12 \, \atbone(J)^2 \le \atb(J) \le \atbone(J).
\end{equation*}
\end{lemma}
\begin{proof}
    Fix a threshold $q$, we write $\Delta_1(q)= \E_J[(v - y)\ind{v < q}]$ and $\Delta_2(q) = \E_J[(v - y)\ind{v \ge q}]$. The right inequality follows from the fact that $\Delta_1, \Delta_2\in [-1, 1]$. 

    Using Jensen's inequality, we get the left inequality:
    \begin{align*}
        \frac{1}{2}\atbone(J)^2 = 2\left(\expect{q}{\frac{1}{2}|\Delta_1(q)| + \frac{1}{2}|\Delta_2(q)|}\right)^2\leq 2\expect{q}{\frac{1}{2}\Delta_1(q)^2 + \frac{1}{2}\Delta_2(q)^2} = \atb(J).
    \end{align*}
\end{proof}

\subsection{Completeness and Soundness}

The completeness and soundness of $\atb$ follows from the quadratic approximation of existing calibration measure Distance to Calibration. 
\begin{theorem}
\label{thm: complete sound}
     Both $\atb$ and $\atbone$ are complete and sound. 
\end{theorem}

Appendix \ref{sec:atbone} will prove that $\atbone$ is a constant approximation of Distance to Calibration. Combined with \Cref{lem: atb 1-atb}, $\atb$ is quadratically related to Distance to Calibration, which implies \Cref{thm: complete sound}.

\subsubsection{Approximating the Distance to Calibration Using Two Bins}
\label[appendix]{sec:atbone}


In this section, we show that $\atbone$ is a constant-factor approximation of both $\smooth$ and $\distcal$ (recall \Cref{prop:dist-smooth} that $\smooth$ and $\distcal$ are constant-factor approximations to each other):
\begin{theorem}
\label{thm:atbone}
    For any distribution $J$ of prediction-state pairs $(v,y)\in [0,1]\times \{0,1\}$, we have
    \[
    \frac 23\, \smooth (J) \le \atbone (J) \le 3\, \distcal (J).
    \]
\end{theorem}

Combining \Cref{thm:atbone} with  \Cref{prop:dist-smooth} and \Cref{lem: atb 1-atb}, we have the following corollary about the relationship between $\atb,\atbone$ and $\smooth,\distcal$:
\begin{corollary}
\label[corollary]{cor:relationship}
    For any distribution $J$ of prediction-state pairs $(v,y)\in [0,1]\times \{0,1\}$, we have
    \begin{align*}
    \frac 13\, \distcal(J) \le \frac 23\, \smooth (J) & \le \atbone (J) \le 3\, \distcal (J) \le 6\, \smooth(J),\\
    \frac 1{18}\, \distcal(J)^2 \le \frac 29\, \smooth (J)^2 & \le \atb (J) \le 3\, \distcal (J) \le 6\, \smooth(J).
    \end{align*}
\end{corollary}

We prove the two inequalities in \Cref{thm:atbone} in two separate lemmas below. We start with the easier one showing the upper bound on $\atbone$:

\begin{lemma}
\label[lemma]{lm:atb-distcal}
   For any distribution $J$ of $(v,y)\in [0,1]\times \{0,1\}$,
    \[
    \atbone(J) \le 3\, \distcal(J).
    \]
\end{lemma}
\begin{proof}
    Let $\Pi$ be an arbitrary distribution of $(u,v,y)\in [0,1]\times [0,1]\times \{0,1\}$, where the distribution of $(v,y)$ is $J$, and the distribution of $(u,y)$ (denoted by $\widehat J$) is calibrated. Since $\widehat J$ is calibrated, we have
    \[
    \atbone(\widehat J) = 0.
    \]
    By \Cref{thm:continuity}, 
    \[
    \atbone(J) = \atbone(J) - \atbone(\widehat J) \le 3 \E_\Pi |u - v|.
    \]
    The lemma is proved by taking the infimum over $\Pi$.
\end{proof}


Now we prove the other inequality in \Cref{thm:atbone} showing the lower bound on $\atbone$. It turns out to be convenient to first focus on the setting with $T$ fixed individuals:
\begin{lemma}
\label[lemma]{thm:smooth-atb}
For any prediction sequence $\vreport\in [0,1]^T$ and any state sequence $\vstate\in \{0,1\}^T$, we have
\begin{equation*}
\smooth(\vreport,\vstate) \le \frac 32 \cdot \atbone(\vreport,\vstate).
\end{equation*}
\end{lemma}


\begin{proof}[Proof of \Cref{thm:smooth-atb}]
It suffices to prove that for any $1$-Lipschitz function $w:[0,1]\to [-1,1]$,
\begin{equation}
\label{eq:atb-smooth}
\frac 1T \sum_{t = 1}^T (\report_t - \state_t) w(\report_t) \le \frac 32\cdot \atbone(\vreport,\vstate).
\end{equation}

Assume without loss of generality that the predictions are sorted: $\report_1 \le \cdots \le \report_T$. Define $w(\report_0) = 0, w(\report_{T + 1}) = 0$. For $t = 0,\ldots,T$, define $\Delta_t:= w(\report_{t+1}) - w(\report_{t})$. We have
\begin{align*}
w(\report_t) = \frac 12 ((w(\report_t) - w(\report_0)) - (w(\report_{T+1}) - w(\report_t))) & = \frac 12 \left(\sum_{s < t} \Delta_{s} - \sum_{s \ge t}\Delta_{s}\right)\\
& = \frac 12 \sum_{s = 0}^{T}\Delta_{s}\sign(t - s),
\end{align*}
where $\sign(u) = 1$ if $u > 0$, and $\sign(u) = -1$ if $u \le 0$.
Therefore,
\begin{equation}
\label{eq:atb-1}
\frac 1T \sum_{t = 1}^T (\report_t - \state_t) w(\report_t) = \frac 1{2T}\sum_{s = 0}^{T}\sum_{t = 1}^T(\report_t - \state_t)\Delta_{s}\sign(t - s).
\end{equation}
For $s = 1,\ldots,T - 1$, by the Lipschitzness of $w$, we have $|\Delta_{s}| \le \report_{s + 1} - \report_{s}$. Therefore,
\begin{align*}
& \left|\frac 1T \sum_{t = 1}^T(\report_t - \state_t)\Delta_{s}\sign(t-s)\right|\\
\le {} & (\report_{s+1} - \report_{s})\left| \frac 1T\sum_{t = 1}^T(\report_t - \state_t)\sign(t - s)\right|\\
\le {} & (\report_{s+1} - \report_{s}) \cdot  \frac 1T \left (\left |\sum_{t \le s}(\report_t - \state_t)\right| + \left |\sum_{t > s}(\report_t - \state_t)\right| \right )\\
= {} & \E_{q\sim \mathrm{Unif}([0,1])}\left[
\mathbb I_{q\in[\report_s, \report_{s + 1}]}\cdot \frac 1T\left(\left|\sum_{t: \report_t < q} (\report_t - \state_t)\right| + \left|\sum_{t: \report_t \ge q}(\report_t - \state_t)\right|\right)
\right].
\end{align*}
Summing up over $s = 1,\ldots T -1$, we have
\begin{equation}
\label{eq:atbone-sm-1}
\sum_{s = 1}^{T - 1}\left|\frac 1T \sum_{t = 1}^T(\report_t - \state_t)\Delta_{s}\sign(t-s)\right| \le \atbone(\vreport,\vstate).
\end{equation}
Moreover, since $w(r_1),w(r_T)\in [-1,1]$, we have $|\Delta_0|,|\Delta_T|\le 1$. Therefore,
\begin{align}
\left|\frac 1T \sum_{t = 1}^T(\report_t - \state_t)\Delta_{0}\sign(t-0)\right| &= |\Delta_0| \cdot \left|\frac 1T \sum_{t = 1}^T(\report_t - \state_t)\right|  \le \atbone(\vreport,\vstate), \label{eq:atbone-sm-2}\\
\left|\frac 1T \sum_{t = 1}^T(\report_t - \state_t)\Delta_{T}\sign(t-T)\right| &= |\Delta_T| \cdot \left|\frac 1T \sum_{t = 1}^T(\report_t - \state_t)\right| \le \atbone(\vreport,\vstate).\label{eq:atbone-sm-3}
\end{align}
Adding up the three inequalities \eqref{eq:atbone-sm-1} \eqref{eq:atbone-sm-2} \eqref{eq:atbone-sm-3} above, we get
\[
\sum_{s = 0}^{T }\left|\frac 1T \sum_{t = 1}^T(\report_t - \state_t)\Delta_{s}\sign(t-s)\right| \le 3\, \atbone(\vreport,\vstate).
\]
Combining this with \eqref{eq:atb-1} using the triangle inequality, we get \eqref{eq:atb-smooth}, as desired.
\end{proof}
\begin{proof}[Proof of \Cref{thm:atbone}]
    The upper bound on $\atbone$ has been proved in \Cref{lm:atb-distcal}. It remains to establish the lower bound on $\atbone$:
    \begin{equation}
    \label{eq:atbone-proof-1}
    \smooth(J) \le \frac 32\cdot \atbone(J).
    \end{equation}
    Consider a sample $S$ of $T$ i.i.d.\ points $(v_1,y_1),\ldots,(v_T,y_T)$ from $J$. Defining $\vreport:= (v_1,\ldots,v_T)$ and $\vstate:= (y_1,\ldots,y_T)$, we have
    \begin{align*}
    \smooth(J_S) & = \smooth(J_{\vreport,\vstate}) = \smooth(\vreport,\vstate),\\
    \atbone(J_S) & = \atbone(J_{\vreport,\vstate}) = \atbone(\vreport,\vstate),
    \end{align*}
where we use the $J_{\vreport,\vstate}$ notation from \Cref{def: induced error}. By \Cref{thm:smooth-atb},
\begin{equation}
\label{eq:atbone-proof-2}
    \smooth(J_S) = \smooth (\vreport,\vstate) \le \frac 32 \cdot \atbone (\vreport,\vstate) = \frac 32 \cdot \atbone(J_S).
\end{equation}
Taking $T \to \infty$, by \Cref{thm:sample} and \Cref{prop:smooth-sample}, we know that $\smooth(J_S)$ converges in probability to $\smooth(J)$, and that $\atbone(J_S)$ converges in probability to $\atbone(J)$. Therefore, our goal \eqref{eq:atbone-proof-1} follows from \eqref{eq:atbone-proof-2}.
\end{proof}




\subsection{(Strict) Truthfulness} 
From its definition \eqref{eq:atb-sequence}, $\atb$ is clearly a special case of UBSE (\Cref{def:ubse}), so its truthfulness follows immediately from \Cref{thm:ubse} and \Cref{thm:strict}.
\begin{theorem}[Truthfulness]
\label{thm: atb truthful}
    The calibration measure $\atb$ is interim truthful (\Cref{def:truthful-sequence}). Moreover, $\atb$ inherits the error decomposition:
\begin{equation*}
    \E_{\vstate \sim \vpred}[\atb(\vreport,\vstate)] = \atb(\vreport,\vpred)  + \frac 1{T^2}\sum_{t = 1}^T p_t(1 - p_t).
\end{equation*}
Moreover, $\atb$ is strictly ex-ante truthful. 
\end{theorem}

\subsection{Continuity} The following theorem establishes the continuity of $\atb$ and $\atbone$ with a general formalization. Both errors change continuously as the predictions change.
\begin{theorem}[Continuity]
\label{thm:continuity}
Let $\Pi$ be a joint distribution of $(v_1,v_2,y)\in [0,1]\times [0,1]\times\{0,1\}$. Let $J_1$ denote the distribution of $(v_1,y)$, and let $J_2$ denote the distribution of $(v_2,y)$. We have
\begin{align}
|\atbone(J_1) - \atbone(J_2)| & \le 3\, \E_\Pi|v_1 - v_2|.\label{eq:continuity-goal-1}\\
|\atb(J_1) - \atb(J_2)| & \le 6\, \E_\Pi|v_1 - v_2|.\label{eq:continuity-goal-2}
\end{align}
\end{theorem}
\begin{proof}
    By \Cref{def: atbone}, we have
    \begin{align*}
    \atbone(J_1) & = \E_{q\sim \mathrm{Unif}([0,1])}\Big[\Big|\E_\Pi\Big[\big(v_1 - y\big)\ind{v_1 < q}\Big]\Big| + \Big|\E_\Pi\Big[\big(v_1 - y\big)\ind{v_1 \ge q}\Big]\Big| \Big],\\
    \atbone(J_2) & = \E_{q\sim \mathrm{Unif}([0,1])}\Big[\Big|\E_\Pi\Big[\big(v_2 - y\big)\ind{v_2 < q}\Big]\Big| + \Big|\E_\Pi\Big[\big(v_2 - y\big)\ind{v_2 \ge q}\Big]\Big| \Big].
    \end{align*}
    We define an intermediate quantity
    \[
    \kappa:= \E_{q\sim \mathrm{Unif}([0,1])}\Big[\Big|\E_\Pi\Big[\big(v_2 - y\big)\ind{v_1 < q}\Big]\Big| + \Big|\E_\Pi\Big[\big(v_2 - y\big)\ind{v_1 \ge q}\Big]\Big| \Big].
    \]
    By the triangle inequality,
    \begin{align}
    |\atbone(J_1) - \kappa| & \le \E_{q\sim \mathrm{Unif}([0,1])}\Big[\E_\Pi\Big[\big|v_1 - v_2\big|\ind{v_1 < q}\Big] + \E_\Pi\Big[\big|v_1 - v_2\big|\ind{v_1 \ge q}\Big] \Big]\notag\\
    & = \E_\Pi|v_1 - v_2|.\label{eq:continuity-1}
    \end{align}
    Similarly, noting that $|v_2 - y| \le 1$, we have
    \begin{align}
    |\atbone(J_2) - \kappa| & \le \E_{q\sim \mathrm{Unif}([0,1])}\Big[2\,\E_\Pi\Big|\ind{v_1 < q} - \ind{v_2 < q}\Big| \Big]\notag \\
    & = 2\E_\Pi\Big[\E_{q\sim \mathrm{Unif}([0,1])}\Big|\ind{v_1 < q} - \ind{v_2 < q}\Big|\Big] \notag \\
    & = 2\E_\Pi|v_1 - v_2|.\label{eq:continuity-2}
    \end{align}
Summing up \eqref{eq:continuity-1} and \eqref{eq:continuity-2} proves \eqref{eq:continuity-goal-1}. A similar strategy proves \eqref{eq:continuity-goal-2}, using one extra observation: the function $u^2$ is $2$-Lipshitz for $u\in [-1,1]$. We omit the details.
    \end{proof}

\subsection{Sample Complexity} Both $\atb$ and $\atbone$ can be estimated within error $\varepsilon$ using $O(1/\varepsilon^2)$ i.i.d.\ examples:
\begin{theorem}[Sample complexity]
\label{thm:sample}
Let $J$ be any distribution of prediction-state pairs $(v,y)\in [0,1]\times \{0,1\}$, and let $S$ be a sample of $T$ i.i.d.\ points $(v_1,y_1),\ldots,(v_T,y_T)$ from $J$.
For $\varepsilon,\delta\in (0,1/3)$, assume $T > C\varepsilon^{-2}\log(1/\delta)$ for a sufficiently large absolute constant $C > 0$. With probability at least $1-\delta$ (over the randomness in the sample $S$),
\begin{align*}
|\atbone(J_S) - \atbone(J))| \le \varepsilon,\\
|\atb(J_S) - \atb(J)| \le \varepsilon.
\end{align*}
\end{theorem}
\begin{proof}
    It suffices to show that with probability at least $1-\delta$, for every $q\in [0,1]$,
    \begin{align*}
    \Big| \E_{J_S} \Big[\big(v - y\big)\ind{v < q}\Big] - \E_{J} \Big[\big(v - y\big)\ind{v < q}\Big]  \Big| & \le \varepsilon/4,\quad \text{and}\\
    \Big| \E_{J_S} \Big[\big(v - y\big)\ind{v \ge q}\Big] - \E_{J} \Big[\big(v - y\big)\ind{v \ge q}\Big]  \Big| & \le \varepsilon/4.
    \end{align*}
By \Cref{prop:uni-conv}, it suffices to prove the following Rademacher complexity bounds for the function families $F = \{f_q\}_{q\in [0,1]}$ and $G = \{g_q\}_{q\in [0,1]}$ where $f_q(v,y) = (v - y)\ind{v < q}$ and $g_q(v,y) = (v - y)\ind{v \ge q}$: for every $(v_1,y_1),\ldots,(v_T,y_T)\in [0,1]\times\{0,1\}$,
\begin{align}
    \calR(F;(v_1,y_1),\ldots,(v_T,y_T)) & \le O\left(\sqrt{\frac{1}{T}}\right), \quad \text{and} \label{eq:rade-1}\\
    \quad  \calR(G;(v_1,y_1),\ldots,(v_T,y_T)) & \le O\left(\sqrt{\frac{1}{T}}\right).\label{eq:rade-2}
\end{align}
Now consider the family $H = \{h_q\}_{q\in [0,1]}$ where $h_q(v,y) = \ind{v < q}$. Clearly, $H$ has VC dimension at most $1$. By \Cref{prop:rad-vc}, we have
\begin{equation}
\label{eq:rade-3}
\calR(H;(v_1,y_1),\ldots,(v_T,y_T)) \le O\left(\sqrt{\frac{1}{T}}\right).
\end{equation}
Observe that $f_q(v_i,y_i) = \eta_i(h_q(v_i,y_i))$ for every $i = 1,\ldots,T$ and $q\in [0,1]$, where $\eta_i$ is the uni-variate function $\eta_i(u) = (v_i - y_i)u$ for $u\in \R$. Since $|v_i - y_i| \le 1$, the function $\eta_i$ is $1$-Lipschitz. Therefore, by \Cref{prop:rad-lip}, inequality \eqref{eq:rade-1} follows from \eqref{eq:rade-3}. Inequality \eqref{eq:rade-2} can be proved similarly.
\end{proof}

We remark that an analogous sample complexity bound for $\smooth$ has been shown by \citet{utc} using a similar analysis: 
\begin{proposition}[\citet{utc}]
\label[proposition]{prop:smooth-sample}
Let $J$ be any distribution of prediction-state pairs $(v,y)\in [0,1]\times \{0,1\}$, and let $S$ be a sample of $T$ i.i.d.\ points $(v_1,y_1),\ldots,(v_T,y_T)$ from $J$.
For $\varepsilon,\delta\in (0,1/3)$, assume $T > C\varepsilon^{-2}\log(1/\delta)$ for a sufficiently large absolute constant $C > 0$. 
With probability at least $1-\delta$ (over the randomness in the sample $S$),
\begin{align*}
|\smooth(J_S) - \smooth(J)| \le \varepsilon. 
\end{align*}
\end{proposition}

\subsection{Computational Efficiency}
\label{sec:compute}
As we show in the following theorem, $\atb$ can be computed and approximated efficiently.
\begin{theorem}
\label{thm:compute}
    Given $\vreport,\vpred\in [0,1]^T$, we can compute $\atb(\vreport,\vpred)$  in time $O(T\log T)$. We can also approximate $\atb(\vreport,\vpred)$ up to arbitrary additive error $\varepsilon > 0$ in time $O(T + 1/\varepsilon)$.
\end{theorem}
The algorithm we use to prove \Cref{thm:compute} is extremely easy to describe and implement. Define 
\[
\Delta_1(q)= \frac{1}{T}\left|\sum_{t: \report_t < q} (\report_t - \pred_t)\right| \quad \text{and} \quad \Delta_2(q) = \frac{1}{T}\left|\sum_{t: \report_t \geq q} (\report_t - \pred_t)\right|.\]
The following algorithm computes $\atb$:
\begin{itemize}
    \item $O(T\log T)$ time: sort predictions in increasing order such that $\report_1\leq \report_2\leq \dots\leq \report_T$. Define $r_0 = 0$ and $r_{T + 1} = 1$.
    \item $O(T)$ time: for $q = r_1,\ldots,r_{T+1}$, calculate $\Delta_1(q)$ by scanning predictions in increasing order. 
    Similarly, calculate $\Delta_2(q)$ by scanning predictions in decreasing order. 
    \item $O(T)$ time: Calculate the expectation over threshold $q$: for $t= 1,\dots,T+1$, sum up $\Delta_1(\report_t)^2 + \Delta_2(\report_t)^2$ with weight $|\report_{t} - \report_{t - 1}|$. 
\end{itemize}
If we allow additive error $\varepsilon \in (0,1)$, by \Cref{thm:continuity}, we can round the predictions $r_1,\ldots,r_T$ to multiples of $\varepsilon/6$ and then compute $\atb$ exactly.
This makes all predictions $r_1,\ldots,r_T$ lie in a finite set $\{0,\varepsilon/6, 2\varepsilon/6,\ldots\}\cap [0,1]$ of size $O(1/\varepsilon)$, so the sorting step can be implemented in time $O(T + 1/\varepsilon)$ by bucket sort.

A similar algorithm computes $\atbone(\vreport,\vpred)$ in $O(T \log T)$ time, or approximates $\atbone(\vreport,\vpred)$ up to error $\varepsilon$ in $O(T + 1/\varepsilon)$ time. We note that currently known algorithms for computing $\smooth$ and $\distcal$ are much more complicated, with the best known running time being $O(T\log^2 T)$ and $O(T^2\log T)$, respectively, even when $O(1/\sqrt T)$ additive error is allowed \citep{hutesting}.


\section{Linear-Time Calibration Tester}
\label{sec:test}
In this section, we show that our $\atb$ and $\atbone$ are both optimally valid (\Cref{def: validity}) for $\smooth$ and $\distcal$. It is fairly straightforward to show that $\atbone$ is $O(1/\sqrt T)$-valid using its constant approximation to $\smooth$ (\Cref{thm:atbone}) and its sample complexity bound (\Cref{thm:sample}). In \Cref{thm: optimal validity} below,  we show that $\atb$ is $O(1/\sqrt T)$-valid as well, and that this is optimal up to constant.

These results imply faster algorithms for solving the calibration testing problem studied by \cite{hutesting}, which requires distinguishing, with large constant success probability, whether a distribution $J$ is perfectly calibrated or has $\distcal(J) > \varepsilon$ given i.i.d.\ data points drawn from $J$.
This can be solved by computing $\atb$ or $\atbone$ on $T = O(1/\varepsilon^2)$ data points and compare the result with the threshold $\beta_T$ in the definition of validity (\Cref{def: validity}). By \Cref{thm:compute}, the running time we need is $O(T\log T)$, which already improves the $O(T\log ^2 T)$ time bound of \cite{hutesting}. Moreover, our Lemmas~\ref{lm:test-cons} and \ref{lm:test-sound} show that it suffices to approximate $\atb$ up to additive error $1/(2T)$, which can be achieved in time $O(T)$ by \Cref{thm:compute}, giving the first linear-time algorithm for calibration testing.

\begin{theorem}
\label{thm: optimal validity}
The calibration measure $\atb$ is $O(\frac{1}{\sqrt T})$-valid w.r.t.\ the reference calibration error $\distcal$. That is, $\atb$ is $\{\gamma_T\}$-valid for some sequence $\gamma_1,\gamma_2,\ldots$  with $\gamma_T = O(1/\sqrt T)$. Moreover, this is optimal up to constant factors: if there exists a $\{\gamma_T\}$-valid calibration error w.r.t.\ $\distcal$, then $\gamma_T = \Omega(1/\sqrt T)$.
\end{theorem}

\Cref{thm: optimal validity} is an immediate corollary of the following Lemmas~\ref{lm:test-cons}, \ref{lm:test-sound}, and \ref{lm:test-lower}.

\begin{lemma}
\label[lemma]{lm:test-cons}
Let $J$ be an arbitrary distribution of prediction-state pairs $(v,y)\in [0,1]\times \{0,1\}$ and assume that $J$ is calibrated. For any $T\in \Z_{>0}$, consider a sample $S$ of $T$ i.i.d.\ points $(v_1,y_1),\ldots,(v_T,y_T)\in [0,1]\times \{0,1\}$ from $J$, and let $J_S$ be the uniform distribution over $S$. We have
\[
\Pr_{S\sim J^T}[\atb(J_S) \le 1/T] \ge 3/4.
\]
\end{lemma}
\begin{proof}
Define $\vreport = (v_1,\ldots,v_T)$ and $\vstate = (y_1,\ldots,y_T)$. It is clear that the distribution $J_{\vreport,\vstate}$ (see \Cref{def: induced error}) is equal to the distribution $J_S$. Therefore,
\[
\atb(J_S) = \atb(\vreport,\vstate).
\]
Since $J$ is calibrated, we have $\E_J[y|v = v_t] = v_t$ for every $t = 1,\ldots,T$. Conditioned on $\vreport = (v_1,\ldots,v_T)$, each $y_t$ is independently distributed as the Bernoulli distribution with mean $v_t$. Thus, we have $\vstate\sim \vreport$ as in \Cref{def:truthful-sequence}. 
Therefore, 
\begin{equation}
\label{eq:test-cons-1}
\Pr_{S}[\atb(J_S) \le 1/T \ |\  v_1,\ldots,v_T] = \Pr_{\vstate\sim \vreport}[\atb(\vreport,\vstate) \le 1/T].
\end{equation}
By \Cref{lem: error decomposition},
\[
\E_{\vstate\sim \vreport}[\atb(\vreport,\vstate)] = \atb(\vreport,\vreport) + \frac 1{T^2}\sum_{t = 1}^Tp_t(1 - p_t) = \frac 1{T^2}\sum_{t = 1}^Tp_t(1 - p_t) \le \frac{1}{4T}.
\]
By Markov's inequality,
\[
\Pr_{\vstate\sim \vreport}[\atb(\vreport,\vstate)\le 1 / T] \ge 3/4.
\]
Plugging this into \eqref{eq:test-cons-1} and taking the expectation over $v_1,\ldots,v_T$ completes the proof.
\end{proof}

\begin{lemma}
\label[lemma]{lm:test-sound}
There exists an absolute constant $C  > 0$ such that the following holds. For any $T\in \Z_{>0}$ and any distribution $J$ of $(v,y)\in [0,1]\times \{0,1\}$ with $\distcal(J) \ge C/\sqrt T$, let $S$ be a sample of $T$ i.i.d.\ points from $J$. Then
    \[
    \Pr_{S\sim J^T}[\atb(J_S) \le 2/T] \le 1/4.
    \]
\end{lemma}
\begin{proof}
By \Cref{thm:sample}, there exists an absolute constant $C' > 0$ such that with probability at least $3/4$ over $S\sim J^T$,
\begin{equation}
\label{eq:test-sound-1}
|\atbone(J_S) - \atbone(J)| \le C'/\sqrt T.
\end{equation}
It remains to show that whenever \eqref{eq:test-sound-1} holds, we have
\[
\atb(J_S) > 2/T.
\]
By \Cref{cor:relationship} and our assumption that $\distcal(J) > C/\sqrt T$, we have $\atbone(J) \ge (C/3)/\sqrt T$. Therefore, whenever \eqref{eq:test-sound-1} holds, we have
\[
\atbone(J_S) \ge (C/3 - C')/\sqrt T.
\]
Assuming $C/3 - C' > 0$ which is guaranteed by a sufficiently large $C$, by \Cref{lem: atb 1-atb}, we have 
\[
\atb(J_S) \ge (1/2)(C/3 - C')^2/T.
\]
The proof is completed by choosing $C$ large enough so that $(1/2)(C/3 - C')^2 > 2$.
\end{proof}


\begin{lemma}
\label[lemma]{lm:test-lower}
Let $\{\gamma_T\}_{T = 1,2,\ldots}$ be a sequence of nonnegative real numbers such that there exists a $\{\gamma_T\}$-valid calibration error $\CAL$ w.r.t.\ $\distcal$. Then $\gamma_T = \Omega(1/\sqrt T)$.
\end{lemma}
\begin{proof}
Let us focus on the choices of $T$ such that $\gamma_T < 1/2$.
We define $J_1$ to be the uniform distribution over $\{(1/2, 0), (1/2,1)\}\subseteq [0,1]\times \{0,1\}$.
We define $J_2$ to be the distribution with probability mass $1/2 - \gamma_T$ on $(1/2,0)$, and the remaining probability mass $1/2 + \gamma_T$  on $(1/2,1)$.

Clearly, $J_1$ is calibrated. We claim that $\distcal(J_2) \ge \gamma_T$. Indeed, consider any coupling distribution $\Pi$ of $(u,v,y)\in [0,1]\times[0,1]\times\{0,1\}$, where $(v,y)$ is distributed as $J_2$, and the distribution of $(u,y)$ is calibrated.
By calibration, $\E[u] = \E[y] = 1/2 + \gamma_T$. Therefore, $\E|u - v| \ge \E[u] - \E[v] = \gamma_T$, implying that $\distcal(J_2) \ge \gamma_T$.

Let $\beta_T\in \R$ be the threshold satisfying the requirement of validity (\Cref{def: validity}). Define
\[
\delta_T:= \accp^\CAL(J_1;T,\beta_T) - \accp^\CAL(J_2;T,\beta_T).
\]
Note that the two acceptance probabilities above are w.r.t.\ the randomness in the samples $S_1\sim J_1^T$ and $S_2\sim J_2^T$, respectively, where $J_1^T$ (resp.\ $J_2^T$) is the joint distribution of $T$ i.i.d.\ points from $J_1$ (resp.\ $J_2$).
A standard argument (e.g.\ via Pinsker's inequality) shows that the total variation distance between $J_1^T$ and $J_2^T$ is $O(\gamma_T\sqrt T)$. Therefore,
\[
\delta_T \le O(\gamma_T\sqrt T).
\]
Validity requires $\liminf_{T\to \infty} \delta_T > 0$. Therefore,
\[
\liminf_{T\to \infty} \gamma_T \sqrt T> 0.
\]
This implies $\gamma_T = \Omega(1/\sqrt T)$.
\end{proof}


\bibliographystyle{econ}
\bibliography{ref}

\appendix

\section{Non-Truthfulness of Known Calibration Measures}
\label{apdx: nontruthful}

In this section, we prove \Cref{thm:avg} showing that condition \eqref{eq:intro-avg} (restated below as condition \eqref{eq:avg}) holds for a broad family of calibration measures: $\ell_\alpha\textup-\ECE, \ell_\alpha \textup-\binECE, \smooth$ and $\ell_\alpha\textup-\distcal$ (\Cref{def:dist-alpha}), where $\alpha \ge 1$ is arbitrary. Condition \eqref{eq:intro-avg} states that, for any realization of the states, reporting the average over predictions is weakly better than reporting truthfully for known calibration measures. By \Cref{remark:strict-non-truthful}, this proves that these calibration measures are not truthful.



\begin{definition}[$\ell_\alpha$-Distance to Calibration]
\label[definition]{def:dist-alpha}
    Let $J$ be a distribution of $(v,y)\in [0,1]\times\{0,1\}$. We define its $\ell_\alpha$-distance to calibration (denoted by $\ell_\alpha\textup-\distcal(J)$) similarly to the definition of $\distcal(J)$ in \Cref{def:distcal}. The only difference is that we change the $\ell_1$ distance $|u - v|$ to $|u - v|^\alpha$:
\[
\ell_\alpha\textup-\distcal(J):= \inf_{\Pi}\E_\Pi[|u - v|^\alpha],
\]
where the infimum is over joint distributions $\Pi$ of $(u,v,y)\in [0,1]\times [0,1]\times \{0,1\}$, where $(v,y)$ is distributed according to $J$, and the distribution of $(u,y)$ is calibrated.
\end{definition}
\begin{theorem}
\label{thm:avg}
Let $\CAL$ be a calibration measure from $\{\ell_\alpha\textup-\ECE, \ell_\alpha\textup-\binECE, \smooth, \ell_\alpha\textup-\distcal\}$, where $\alpha \ge 1$ is arbitrary. For every $\vreport = (r_1,\ldots,r_T)\in [0,1]^T$ and every $\vstate\in \{0,1\}^T$, it holds that
\begin{equation}
\label{eq:avg}
\CAL(\bar \vreport,\vstate) \le \CAL (\vreport,\vstate),
\end{equation}
where $\bar \vreport := (\bar r,\ldots,\bar r)\in [0,1]^T$ for $\bar r:= \frac 1T\sum_{t = 1}^Tr_t$.
\end{theorem}
\begin{remark}
\label[remark]{remark:strict-non-truthful}
It is very easy to find (many) examples of $\vstate$ where the inequality \eqref{eq:avg} becomes strict, in which case reporting $\bar \vreport$ instead of the $\vreport$ is strictly better (i.e.\ $\bar \vreport$ dominates $\vreport$). In particular, we can find many examples of $\vstate$ where $\CAL(\vreport,\vstate) > 0$  and $\CAL(\bar \vreport,\vstate) = 0$. This is because for all of the calibration measures $\CAL$ mentioned above (including the continuous ones), we always have $\CAL(\vreport,\vstate) > 0$ as long as $r_1,\ldots,r_T$ are distinct values\footnote{For binned calibration error, we need $\max_{t} r_t - \min_{t}r_t$ to be sufficiently large so that $r_1,\ldots,r_T$ do not fall in the same bin.} in, say, $[1/3,2/3]$,  and we have $\CAL(\bar \vreport,\vstate) = 0$ as long as the average outcome $\frac 1T\sum_{t = 1}^T y_t$ is equal to $\bar r$.
\end{remark}

\Cref{thm:avg} is a corollary of the following theorem.

\begin{theorem}
\label{thm: avg dist version}
Let $\CAL$ be a calibration measure from $\{\ell_\alpha\textup-\ECE, \ell_\alpha\textup-\binECE, \smooth, \ell_\alpha\textup-\distcal\}$, where $\alpha \ge 1$ is arbitrary. 
    Let $J$ be an arbitrary distribution of $(v,y)\in [0,1]\times \{0,1\}$. We have
\begin{equation}
\label{eq:J-bar}
\CAL(\bar J) \le \CAL (J).
\end{equation}
Here $\bar J$ is the distribution of $(\bar v,y)\in [0,1]\times \{0,1\}$, where we draw $(v,y)\sim J$ and replace $v$ with the deterministic quantity $\bar v:= \E_J[v]$.
\end{theorem}

\begin{proof}
    We prove the theorem separately for each choice of $\CAL$.
Similarly to the definition of $\bar v$, we define $\bar y:= \E_J[y]$.

When $\CAL = \ell_\alpha\textup-\ECE$, defining $\hat v:= \E[y|v]$, by Jensen's Inequality we have
\begin{equation}
\label{eq:avg-ece-1}
\ECE(J)= \E[|v - \hat v|^\alpha] \ge |\E[v - \hat v]|^\alpha = |\bar v - \bar y|^\alpha,
\end{equation}
where we used the fact that $\E[\hat v] = \E[y] = \bar y$.
Also,
\begin{equation}
\label{eq:avg-ece-2}
\ECE(\bar J) = \E[|\bar v - \E[y|\bar v]|^\alpha] = |\bar v - \bar y|^\alpha,
\end{equation}
where we used the fact that $\bar v$ is a deterministic quantity, so $\E[y|\bar v] = \E[y] = \bar y$.
Combining \eqref{eq:avg-ece-1} and \eqref{eq:avg-ece-2} proves \eqref{eq:J-bar}.

When $\CAL = \ell_\alpha\textup-\binECE$, we can prove \eqref{eq:J-bar} as follows. Let $\intvset = \{\intv_i\}_{i\in [k]}$ be the partition  of the prediction space $[0, 1]$ in the definition of $\ell_\alpha\textup-\binECE$ (\Cref{def:bin}). By Jensen's Inequality,
\begin{align}
    \ell_\alpha\textup-\binECE(J) = \sum_{i\in [k]}\Pr\nolimits_{J}[v\in I_i]\cdot \big|\E_{J}[v - y|v\in I_i]\big|^\alpha & \ge \left|
    \sum_{i\in [k]}\Pr\nolimits_{J}[v\in I_i]\cdot \E_{J}[v - y|v\in I_i]\right|^\alpha\notag \\
    & = |\E_J[v - y]|^\alpha\notag\\
    & = |\bar v - \bar y|^\alpha.\label{eq:avg-bin-1}
\end{align}
Also, since $\bar v$ is deterministic, we have
\begin{equation}
\label{eq:avg-bin-2}
    \ell_\alpha\textup-\binECE(\bar J) = \sum_{i\in [k]}\Pr[\bar v\in I_i]\cdot \big|\E[\bar v - y|\bar v\in I_i]\big|^\alpha = |\E[\bar v - y]|^\alpha = |\bar v - \bar y|^\alpha.
\end{equation}
Combining \eqref{eq:avg-bin-1} and \eqref{eq:avg-bin-2} proves \eqref{eq:J-bar}.

When $\CAL = \smooth$, \eqref{eq:J-bar} follows from the following calculation:
\begin{align*}
    \smooth(\bar J) & = \sup_{w\in W_1}\E[(\bar v - y)w(\bar v)] = \sup_{w\in W_1}(\bar v - \bar y)w(\bar v) = |\bar v - \bar y|,\\
    \smooth(J)  & = \sup_{w\in W_1}\E[(v - y)w(v)] \ge \sup_{\sigma \in \{\pm 1\}} \E[(v - y)\sigma] = \sup_{\sigma\in \{\pm 1\}}(\bar v - \bar y)\sigma = |\bar v - \bar y|.
\end{align*}

When $\CAL = \ell_\alpha\textup-\distcal$, \eqref{eq:J-bar} holds because of the following argument. Consider any joint distribution $\Pi$ of $(u,v,y)\in [0,1]\times  [0,1]\times \{0,1\}$, where the marginal distribution of $(v,y)$ is $J$, and the marginal distribution of $(u,y)$ is calibrated. By Jensen's Inequality,
\[
    \E_\Pi[|u - v|^\alpha] \ge |\E_\Pi[u - v]|^\alpha = |\E_\Pi[u] -\E_\Pi[v]|^\alpha = |\E_{\Pi}[y] - \E_{\Pi}[v]]|^\alpha = |\bar y - \bar v|^\alpha.
\]
We used the fact that $\E_\Pi[u] = \E_\Pi[y]$, which holds because the distribution of $(u,y)$ is calibrated.
Taking the infimum over $\Pi$, we get 
\begin{equation}
\label{eq:avg-distcal-1}
\distcal(J) \ge |\bar y - \bar v|^\alpha.
\end{equation}
Moreover, since always predicting $\bar y$ yields a calibrated predictor, we can choose $\Pi$ to be the joint distribution of $(\bar y,\bar v,y)$ and get
\begin{equation}
\label{eq:avg-distcal-2}
\distcal(\bar J)  \le \E_{\Pi}[|\bar y - \bar v|^\alpha] =  |\bar v - \bar y|^\alpha.
\end{equation}
Combining \eqref{eq:avg-distcal-1} and \eqref{eq:avg-distcal-2} proves \eqref{eq:J-bar}.
%
\end{proof}

\Cref{tab: two state non truthful} provides an example illustrating the non-truthfulness of known calibration measures and the truthfulness of our $\atb$. The table compares two strategies: predicting the overall average and predicting truthfully.

\setlength{\tabcolsep}{3pt}

\newcommand{\ra}[1]{\renewcommand{\arraystretch}{#1}}
\begin{table*}\centering
\ra{1.3}
\begin{tabular}{ll|ll|ll|ll|ll}\toprule
\multirow{2}{2.5em}{Prob.}& \multirow{2}{0em}{States} & \multicolumn{2}{l|}{$\smooth$} & \multicolumn{2}{l|}{$\distcal$}& \multicolumn{2}{l|}{$\distcaltwo$} & \multicolumn{2}{l}{$\atb$ (ours)} \\
& & \texttt{avg} & truth & \texttt{avg} & truth & \texttt{avg} & truth & \texttt{avg} & truth\\
\midrule
$\frac 3{16}$ & $(0,0)$ & 0.5 & 0.5   & 0.5   & 0.5 & 0.25 & 0.3125 & 0.25 & 0.203125 \\
$\frac 3{16}$ & $(1,1)$ & 0.5 & 0.5   & 0.5   & 0.5 & 0.25 & 0.3125 & 0.25 & 0.203125\\
$\frac 9{16}$ & $(0,1)$ & 0   &  $0.0625$  & 0     & $> 0$ & 0 & $> 0$ & 0 & 0.015625\\
$\frac 1{16}$ & $(1,0)$ & 0   &  $0.1875$  & 0     & $> 0$ & 0 & $> 0$ & 0 & 0.140625 \\
\midrule
\multicolumn{2}{l|}{Expected Error} & 0.1875 & 0.234375 & 0.1875 & $> 0.1875$ & 0.09375 & $> 0.11$ & \textbf{0.09375} & \textbf{0.09375} \\
\bottomrule
\end{tabular}
\caption{The calibration errors of predictors with two data points. The ground truth probabilities of the two points are $25\%$ and $75\%$, respectively. In the table, \texttt{avg} stands for the uninformative predictor that always outputs $50\%$ and \texttt{truth} stands for the truthful predictor that outputs $25\%$ and $75\%$. We calculate the error of the predictors given each realization of the state and the total expected error. For non-truthful error metrics, the expected error of a truthful predictor is strictly higher than the expected error of an uninformative predictor. For $\atb$, the expected errors are the same. }
\label{tab: two state non truthful}
\end{table*}

\section{$\atb$ and Brier Loss}
\label{apdx: atb brier}

For ATB as a special case of UBSE, we have the following stronger result than \Cref{lm:ubse-brier}:
\begin{lemma}[ATB and Brier loss]
\label[lemma]{lm:atb-brier}
    Let $J$ be an arbitrary distribution of $(v,y)\in [0,1]\times \{0,1\}$. For $(v_1,y_1),\ldots,(v_T,y_T)$ drawn i.i.d.\ from $J$, defining $\bm v:= (v_1,\ldots,v_T),\vstate = (y_1,\ldots,y_T)$, we have
    \begin{equation}
    \label{eq:atb-brier}
    \E[\atb(\bm v,\vstate)] \ge \frac 1T \E_{(v,y)\sim J}[(v - y)^2].
    \end{equation}
The inequality becomes an equality if $J$ is perfectly calibrated.
\end{lemma}
\begin{proof}
    For every fixed threshold $q\in [0,1]$, we have
    \begin{align*}
        \E\left[\left(\sum_{t:v_t < q}(v_t - y_t)\right)^2\right] & = \E\left[\left(\sum_{t=1}^T(v_t - y_t)\ind{v_t < q}\right)^2\right]\\
        & = T \E_J[((v - y)\ind{v < q})^2] + T(T - 1)\E_J[(v - y)\ind{v < q}]^2\tag{because $(v_1,y_1),\ldots,(v_T,y_T)$ are drawn i.i.d.\ from $J$}\\
        & \ge T\E_J[((v - y)\ind{v < q})^2]\\
        & = T \E_J[(v - y)^2\ind{v < q}].
    \end{align*}
Similarly,
\[
\E\left[\left(\sum_{t:v_t \ge q}(v_t - y_t)\right)^2\right] \ge T \E_J[(v - y)^2\ind{v \ge  q}].
\]
Summing up the two inequalities above, for every $q\in [0,1]$ we have
\[
\frac 1{T^2}\E\left[\left(\sum_{t:v_t < q}(v_t - y_t)\right)^2 + \left(\sum_{t:v_t \ge q}(v_t - y_t)\right)^2\right] \ge \frac 1T\E_J[(v - y)^2].
\]
Taking expectation over $q\sim \mathrm{Unif}([0,1])$ proves \eqref{eq:atb-brier}. When $J$ is perfectly calibrated, all inequalities in this proof become equalities, so \eqref{eq:atb-brier} also holds as an equality.
\end{proof}

\section{Standard Uniform Convergence Bounds}

We include some standard notions and results on concentration inequalities and sample complexity bounds for uniform convergence.
They are used when we prove the sample complexity bounds for estimating $\atb$ and $\atbone$ in \Cref{thm:sample}.

We start with the definition of the Rademacher complexity.
\begin{definition}[Rademacher complexity]
\label[definition]{def:rademacher}
Let $\calF$ be a family of real-valued functions $f:\zset\to \R$ on some domain $\zset$. Given $z_1,\ldots,z_n\in \zset$, we define the Rademacher complexity as follows:
\[
\calR(\calF;z_{1,\ldots, n}) := \E\left[\sup_{f\in \calF} \frac 1n \sum_{i=1}^ns_if(z_i)\right],
\]
where the expectation is over $s_1,\ldots,s_n$ drawn uniformly at random from $\{-1,1\}^n$.
\end{definition}
The following theorem is a standard application of the Rademacher complexity for proving uniform convergence bounds.
\begin{proposition}[Uniform convergence from Rademacher complexity]
\label[proposition]{prop:uni-conv}
Let $\calF$ be a family of functions $f:\zset\to [a,b]$ on some domain $\zset$ and with range bounded in $[a,b]$. Let $\calD$ be an arbitrary distribution over $\zset$. Then for $n$ i.i.d.\ examples $z_1,\ldots,z_n$ from $\calD$,
\[
\E_{z_{1,\ldots,n}}\left[\sup_{f\in \calF}\left|\frac 1n \sum_{i=1}^n f(z_i) - \E_{z\sim \calD}[f(z)]\right|\right] \le 2\E_{z_{1,\ldots,n}}[\calR(\calF;z_{1,\ldots,n})].
\]
Moreover, for any $\delta\in (0, \frac 1 3)$ and $n \in \N$, with probability at least $1-\delta$ over the random draw of $n$ i.i.d.\ examples $z_1,\ldots,z_n$ from $\calD$, it holds that
\[
\sup_{f\in \calF}\left|
\frac 1n \sum_{i=1}^n f(z_i) - \E_{z\sim \calD}[f(z)]
\right| \le 2\calR(\calF;z_{1,\ldots,n}) + O\left((b - a)\sqrt{\frac{\log(1/\delta)}{n}}\right).
\]
\end{proposition}

\begin{proposition}[Rademacher Complexity after Lipschitz Postprocessing]
\label[proposition]{prop:rad-lip}
Let $\calF$ be a family of functions $f:\zset \to \R$. For $i = 1,\ldots,n$, let $z_i\in \zset$ be an element of the domain $\zset$ and let $\eta_i:\R\to \R$ be any $1$-Lipschitz function. It holds that
\[
\E\left[\sup_{f\in \calF}\frac 1n\sum_{i=1}^n s_i\eta_i(f(z_i))\right] \le \calR(\calF;z_{1,\ldots,n}) = \E\left[\sup_{f\in \calF}\frac 1n\sum_{i=1}^n s_if(z_i)\right].
\]
\end{proposition}
\begin{proof}
By induction, it suffices to consider the case where all the $\eta_i$'s are the identity function except $\eta_1$.

Now we have
\begin{align}
& \E\left[\sup_{f\in \calF}\frac 1n\sum_{i=1}^n s_i\eta_i(f(z_i))\right]\notag\\
= {} & \frac 1{2n}\E\left[\sup_{f\in \calF}\left(s_1\eta_1(f(z_1)) + \sum_{i = 2}^n s_if(z_i)\right) + \sup_{f\in \calF}\left(-s_1\eta_1(f(z_1)) + \sum_{i = 2}^n s_if(z_i)\right)\right]\notag\\
= {} & \frac 1{2n}\E\left[\sup_{f_+,f_-\in \calF}\left(\eta_1(f_+(z_1)) - \eta_1(f_-(z_1)) + \sum_{i=2}^ns_i(f_+(z_i) + f_-(z_i))\right)\right]\notag\\
= {} & \frac 1{2n}\E\left[\sup_{f_+,f_-\in \calF}\left(|\eta_1(f_+(z_1)) - \eta_1(f_-(z_1))| + \sum_{i=2}^ns_i(f_+(z_i) + f_-(z_i))\right)\right].\label{eq:rad-lip-1}
\end{align}
Similarly,
\begin{align}
& \E\left[\sup_{f\in \calF}\frac 1n\sum_{i=1}^n s_if(z_i)\right]\notag\\
= {} & \frac 1{2n}\E\left[\sup_{f_+,f_-\in \calF}\left(|f_+(z_1) - f_-(z_1)| + \sum_{i=2}^ns_i(f_+(z_i) + f_-(z_i))\right)\right]. \label{eq:rad-lip-2}
\end{align}
By the $1$-Lipschitz property of $\eta_1$, we have
\[
|\eta_1(f_+(z_1)) - \eta_1(f_-(z_1))| \le |f_+(z_1) - f_-(z_1)|.
\]
This implies that \eqref{eq:rad-lip-1} is a lower bound of \eqref{eq:rad-lip-2}, completing the proof.
\end{proof}

The following is the standard definition of the VC dimention for binary function families:

\begin{definition}[VC Dimension \citep{vc}]
    The VC dimension of a family $F$ of binary functions $f:Z\to \{0,1\}$ is the largest size $d$ of a subset $Z' = \{z_1,\ldots,z_d\}\subseteq Z$ such that for each of the $2^d$ choices of 
    $\sv:=(s_1,\ldots,s_d)\in \{0,1\}^n$, there exists $f_{\sv}\in F$ such that $f_{\sv}(z_i) = s_i$ for every $i = 1,\ldots,d$.
\end{definition}

The following standard result can be proved using Dudley's chaining argument (see e.g.\ Theorem 8.3.23 of \citet{vershynin}):
\begin{proposition}[Rademacher Complexity from VC Dimension]
\label[proposition]{prop:rad-vc}
Let $F$ be a family of binary functions $f:Z\to \{0,1\}$ with VC dimension at most $d$. Then for any $n\in \Z_{>0}$ and any $z_1,\ldots,z_n\in Z$, we have
\[
\calR(F;z_{1,\ldots,n}) \le O\left(\sqrt{\frac dn}\right).
\]
\end{proposition}

\end{document}